\documentclass[11pt]{article}

\usepackage[a4paper,margin=1in]{geometry}
\usepackage{setspace}
\usepackage{amsmath,amssymb,amsthm,mathtools}
\usepackage{algorithm, algorithmicx, algpseudocode}
\usepackage{bm}

\theoremstyle{plain}
\newtheorem{theorem}{Theorem}
\newtheorem{proposition}[theorem]{Proposition}
\newtheorem{lemma}[theorem]{Lemma}

\theoremstyle{definition}
\newtheorem{definition}[theorem]{Definition}
\newtheorem{assumption}{Assumption}

\theoremstyle{remark}

\usepackage{enumitem}
\usepackage{graphicx}
\usepackage{booktabs}
\usepackage{array}
\usepackage{graphicx}
\usepackage{caption}
\usepackage{subcaption}
\usepackage{xcolor}
\usepackage{multirow}
\usepackage{hyperref}
\hypersetup{
    colorlinks=true,
    linkcolor=black,
    citecolor=black,
    urlcolor=black
}

\title{\textbf{Random Coordinate Descent on the Wasserstein Space of Probability Measures}}

\author{
    Yewei Xu\\[0.3em]
    \small Department of Mathematics\\
    \small University of Wisconsin-Madison\\
    \small Madison, WI 53706-1388, USA\\
    \small \texttt{xu464@wisc.edu}
    \and
    Qin Li\\[0.3em]
    \small Department of Mathematics\\
    \small University of Wisconsin-Madison\\
    \small Madison, WI 53706-1388, USA\\
    \small \texttt{qinli@math.wisc.edu}
}

\date{}

\begin{document}

\maketitle

\begin{abstract}
Optimization over the space of probability measures endowed with the Wasserstein-2 geometry is central to modern machine learning and mean-field modeling. However, traditional methods relying on full Wasserstein gradients often suffer from high computational overhead in high-dimensional or ill-conditioned settings. We propose a randomized coordinate descent framework specifically designed for the Wasserstein manifold, introducing both Random Wasserstein Coordinate Descent (RWCD) and Random Wasserstein Coordinate Proximal{-Gradient} (RWCP) for composite objectives. By exploiting coordinate-wise structures, our methods adapt to anisotropic objective landscapes where full-gradient approaches typically struggle. We provide a rigorous convergence analysis across various landscape geometries, establishing guarantees under non-convex, Polyak-{\L}ojasiewicz, and geodesically convex conditions. Our theoretical results mirror the classic convergence properties found in Euclidean space, revealing a compelling symmetry between coordinate descent on vectors and on probability measures. The developed techniques are inherently adaptive to the Wasserstein geometry and offer a robust analytical template that can be extended to other optimization solvers within the space of measures. Numerical experiments on ill-conditioned energies demonstrate that our framework offers significant speedups over conventional full-gradient methods.
\end{abstract}


\section{Introduction}
Optimization over the space of probability measures has emerged as a unifying framework across numerous disciplines. In the physical sciences, stationary states of complex systems are often characterized as probability measures that minimize specific energy functionals. This perspective is foundational to the stability analysis of the Vlasov-Poisson system in plasma physics~\cite{R01, R94, M81}, evolution of crystal surfaces~\cite{CLLMW22}, classical density functional theory for fluids~\cite{EORK16}, and density functional theory in quantum mechanics~\cite{LS09}. 
Beyond physics, this optimization paradigm also plays a critical role in Bayesian inference and sampling~\cite{LW16, W18, LCBBR22, THS23, LSW23, HBD24, BPBMN25, CHHRS26}, uncertainty quantification~\cite{WCL22}, optimal transport~\cite{CHBM25}, generative modeling~\cite{ACB17, C22, ABBGHLPSTT23, LSSVW25}, shallow neural network training~\cite{CB18, MMN18, SS20, DCLW22, CLTW25}, and inverse problems~\cite{CDDJ24}.
These diverse applications underscore the fundamental importance of developing efficient numerical methods for problems posed over the space of probability measures.

The general problem is formulated as:
\begin{equation}\label{eq:intro_problem}
\min_{\mu \in \mathcal{P}(\mathbb{R}^d)} E[\mu]\,,
\end{equation}
where $\mathcal{P}(\mathbb{R}^d)$ denotes the space of probability measures on $\mathbb{R}^d$. This formulation directly parallels the classical optimization problem in Euclidean space:
\begin{equation}\label{eq:min_f}
\min_{x \in \mathbb{R}^d} f(x)\,.
\end{equation}
A natural question arises: to what extent can established algorithmic paradigms from Euclidean optimization be ported to the setting of probability measures?

Over the past decade, an extensive dictionary has emerged linking Euclidean optimization with its counterpart in the space of probability measures. The most celebrated example is the Wasserstein gradient flow (WGF), which serves as the natural analogue to Euclidean gradient descent (GD)~\cite{AGS08, CNR25}. In this framework, the measure $\mu$ is updated along the negative Wasserstein gradient, mirroring the displacement of $x$ in GD. Similar parallels have been established for more complex methods: Hamiltonian-type flows correspond to momentum-based acceleration~\cite{CLTW25,WLYZ25}, the Jordan--Kinderlehrer--Otto (JKO) scheme acts as a proximal or implicit step~\cite{JKO98}, and splitting schemes represent proximal-gradient variants in Wasserstein space~\cite{SKL20, BUDAK24, JLL21, DBCS23}. This perspective has proven remarkably powerful, yielding a suite of elegant solvers derived directly from Euclidean principles.

Despite these advances, existing algorithms typically rely on full Wasserstein gradient information at every iteration. This requirement imposes a significant computational burden, especially in high-dimensional settings, making the development of more efficient alternatives highly desirable. A natural progression is to lift the concepts of Random Coordinate Descent (RCD) and the Random Coordinate Proximal{-Gradient} (RCP) method to the Wasserstein manifold. In Euclidean space, these methods are designed to reduce per-iteration complexity and accelerate global convergence by exploiting the coordinate-wise structure of the objective.

The advantages of RCD are particularly evident when a functional $f$ is $m$-strongly convex and $L_i$-smooth with respect to each coordinate $i$. Classical analysis~\cite{N12} reveals that to achieve an $\epsilon$-accurate solution, the total computational costs are:
\[
\mathrm{Cost(GD)}=\mathcal{O}\!\Big(\frac{dL}{m}\ln\frac{1}{\epsilon}\Big),
\qquad
\mathbb E\,\mathrm{Cost(RCD)}=\mathcal{O}\!\Big(\frac{L_{\mathrm{sum}}}{m}\ln\frac{1}{\epsilon}\Big)\,.
\]
RCD outperforms GD whenever $L_{\mathrm{sum}} := \sum_i L_i < dL$, with the performance gain becoming more pronounced in anisotropic cases where $L_{\mathrm{sum}} \ll dL$. This efficiency stems from a fundamental constraint of GD: its step size is restricted by the global smoothness constant $L \ge \max_i L_i$. In contrast, coordinate-wise updates allow for larger, adaptive step sizes on the order of $1/L_i$. When the Lipschitz constants $L_i$ are highly non-uniform, with a few ``stiff" directions dominate the global constant $L$, GD is forced to take unnecessarily small steps across all coordinates, leading to substantial computational waste that RCD effectively avoids.

We anticipate that a similar phenomenon occurs in the context of Wasserstein gradients. However, lifting Euclidean optimization algorithms to the space of probability measures introduces several non-trivial challenges. Two of these are inherent to almost all Wasserstein solvers: (i) the space $\mathcal{P}(\mathbb{R}^d)$ lacks a natural linear structure, requiring the imposition of a metric and an associated differential structure. Throughout this work, we focus on $\mathcal{P}_2(\mathbb{R}^d)$ equipped with the Wasserstein-2 geometry; and (ii) standard optimization methods developed for finite-dimensional Riemannian manifolds typically rely on smooth-manifold machinery, which the Wasserstein space does not strictly satisfy. Beyond these, the primary challenge in developing a Random Coordinate framework in $\mathcal{P}_2$ is defining what it means to ``update a single coordinate" at the level of probability measures. In Euclidean optimization, a coordinate descent step updates one component of a $d$-dimensional vector while leaving all others unchanged. In contrast, a probability measure lacks an immediate, canonical counterpart to the components of a vector.

Fortunately, probability measures possess additional structure that can be exploited algorithmically. A standard computational perspective is to represent $\mu$ via an ensemble of particles, $\mu \approx \frac{1}{N}\sum_{n=1}^N \delta_{x_n}$, such that updates to the measure are realized through updates to its representative particles. Specifically, updating a probability measure via a pushforward, $\mu \mapsto (\mathrm{Id}+T)_\#\mu$, corresponds to the particle update $x \to x + T(x)$, where $v: \mathbb{R}^d \to \mathbb{R}^d$ is a vector-valued map. This perspective allows us to reformulate the optimization problem over $\mathcal{P}_2$ as the task of designing specific pushforward maps. To capture ``coordinate" information, we restrict the pushforward map $v$ to the form $U_i T$, where $T$ is a general transport map and $U_i$ is the projection matrix onto the $i$-th coordinate. For example, mirroring RCD, we propose to set $T = -\gamma \nabla_{\mathbb{W}}E$ (the Wasserstein gradient), resulting in the update:
\[
    \mu_{k+1} = (\mathrm{Id} - \gamma_{i_k} U_{i_k} \nabla_{\mathbb{W}}E[\mu_k])_\#\mu_k\,,
\]
where $i_k$ is the coordinate chosen at this step. By construction, such an update modifies only the $i_k$-th marginal of the measure while leaving all other marginals unchanged, thereby constituting a faithful analogue to Euclidean coordinate descent.

The objective of this paper is to formally derive the Wasserstein counterparts for both RCD and RCP and rigorously analyze their convergence properties. In particular, we investigate whether these randomized methods accelerate convergence over standard Wasserstein Gradient Descent (WGD), mirroring the efficiency gains observed in Euclidean space.

\paragraph{Contributions.}
We propose and analyze a randomized coordinate optimization framework over $\mathcal{P}_2(\mathbb{R}^d)$ using coordinate-restricted pushforward updates. Our main contributions are:
\begin{itemize}
\item \textbf{Algorithmic Framework:} We introduce \emph{Random Wasserstein Coordinate Descent} (RWCD) for smooth functionals and \emph{Random Wasserstein Coordinate Proximal{-Gradient}} (RWCP) for composite objectives. Each iteration applies a transport map restricted to a randomly selected coordinate direction. By determining sampling probabilities through directional smoothness, our methods allow for aggressive updates along ``stiff" dimensions while preserving all other marginal distributions, effectively adapting to the landscape's anisotropy.

\item \textbf{Convergence Theory:}
We establish convergence guarantees that mirror the classical complexity bounds of Euclidean optimization. Under coordinate-wise Wasserstein regularity assumptions, we derive rates for stationary point convergence (non-convex), sublinear rates for geodesically convex landscapes, and linear convergence under the PL condition or strong geodesic convexity (Theorems~\ref{thm:rcd_conv} and~\ref{thm:rcp_conv}).

While the resulting rates appear familiar, the underlying proofs are highly non-trivial due to the non-linear metric structure of $\mathcal{P}_2$. Standard Euclidean concepts, such as convexity, $L$-smoothness, and coordinate-wise increments, must be reformulated as geodesic convexity, directional Wasserstein smoothness, and restricted pushforward maps. For example, the linear inner product $\langle \nabla f(y), y-x \rangle$ commonly used in $\mathbb{R}^d$ must be projected and integrated along displacement interpolants in the Wasserstein space, behaving as:
$$\int_{\mathbb{R}^d \times \mathbb{R}^d} \langle \nabla_{\mathbb{W}} E[\mu](x), y - x \rangle \, d \pi_o(x,y)$$
where $\pi_o$ is the optimal transport plan between two measures (to be made precise in Section~\ref{sec:prelim}). This transition necessitates a systematic rewrite of the fundamental descent lemmas. Furthermore, we provide a rigorous characterization of coordinate-wise smoothness for several widely used energy functionals in Proposition~\ref{prop:examples_coord_smooth}.
\item \textbf{Numerical Evidence:}
Through a series of high-dimensional experiments, we demonstrate that RWCD and RWCP provide significant acceleration over standard Wasserstein Gradient Flow (WGF) in ill-conditioned landscapes. Beyond simple validation cases, we showcase the efficiency of our methods on high-dimensional MMD-type functionals and two-layer {neural network (NN) training}, which are of particular relevance to kernel-based particle methods and machine learning.
\end{itemize}

A summary of these convergence rates, illustrating the parallelism between the Euclidean and Wasserstein settings, is provided in Table~\ref{tab:eucl_wass_compact}.

\begin{table}[htb]
\centering
\small
\setlength{\tabcolsep}{5pt}
\renewcommand{\arraystretch}{1.2}
\caption{Euclidean vs.\ Wasserstein convergence rates. Our Wasserstein coordinate methods recover the Euclidean complexity regimes with geometry-specific constants $L_{\rm sum}$ and $\mathcal{C}$ (cf.\ Theorems~\ref{thm:rcd_conv} and~\ref{thm:rcp_conv}).}
\label{tab:eucl_wass_compact}
\begin{tabular}{l cc cc cc}
\toprule
\multirow{3}{*}{Method} 
& \multicolumn{2}{c}{Nonconvex}
& \multicolumn{2}{c}{(Geo.) Convex} 
& \multicolumn{2}{c}{$m$-PL / Strongly Convex} \\
\cmidrule(lr){2-3}\cmidrule(lr){4-5}\cmidrule(lr){6-7}
& Eucl. & Wass.
& Eucl. & Wass. 
& Eucl. & Wass. \\
& $\| \nabla f(x) \|^2$ & $\| \nabla_{\mathbb{W}} E[\mu] \|_{\mu}^2$
& $f(x)-f_*$ & $E[\mu] - E_*$ 
& $f(x)-f_*$ & $E[\mu] - E_*$ \\
\midrule
RCD/RWCD
& $O\left(\frac{L_{\rm sum}}{\epsilon}\right)$ 
& $O\left(\frac{L_{\rm sum}}{\epsilon}\right)$
& $O\left(\frac{L_{\rm sum}}{\epsilon}\right)$ 
& $O\left(\frac{L_{\rm sum}}{\epsilon}\right)$ 
& $O\left(\frac{L_{\rm sum}}{m} \ln\frac{1}{\epsilon}\right)$ 
& $O\left(\frac{L_{\rm sum}}{m} \ln\frac{1}{\epsilon}\right)$ \\
RCP/RWCP
& $O\left(\frac{\mathcal{C}}{\epsilon}\right)$ 
& $O\left(\frac{\mathcal{C}}{\epsilon}\right)$ 
& $O\left(\frac{\mathcal{C}}{\epsilon}\right)$ 
& $O\left(\frac{\mathcal{C}}{\epsilon}\right)$ 
& $O\left(\frac{\mathcal{C}}{m} \ln\frac{1}{\epsilon}\right)$ 
& $O\left(\frac{\mathcal{C}}{m} \ln\frac{1}{\epsilon}\right)$ \\
\bottomrule
\end{tabular}
\end{table}

\paragraph{Organization.}
The remainder of this paper is organized as follows: Section~\ref{sec:prelim} reviews Wasserstein geometry and introduces our coordinate-wise smoothness framework. Sections~\ref{sec:rcd} and~\ref{sec:rcp} present the RWCD and RWCP algorithms along with their respective theoretical analyses. Finally, Section~\ref{sec:num} provides detailed numerical studies.

\subsection{Related Work}\label{sec:related_work}
\begin{itemize}

\item \textbf{Optimization on Wasserstein Space:} 
Substantial progress has been made in characterizing gradient flows on the density manifold $\mathcal{P}_2(\mathbb{R}^d)$. Foundational continuous-time analysis is provided by \cite{AGS08, CNR25}. Discrete-time schemes typically rely on the proximal point (JKO) approach \cite{JKO98} or splitting techniques such as Wasserstein {proximal-gradient method} \cite{SKL20, BUDAK24,JLL21}. While these methods perform global updates in the Wasserstein geometry, our work focuses on \emph{randomized coordinate-restricted} transport updates, which offer computational advantages in high-dimensional settings.

\item \textbf{Coordinate Descent in Euclidean Spaces:} 
Coordinate descent is a cornerstone of Euclidean optimization, favored for its efficiency and implicit preconditioning effects \cite{N12, W15}. For composite objectives, coordinate proximal and proximal-gradient methods are widely employed, typically under the assumption of separable regularizers \cite{FR15, CERS18}. A particularly relevant direction is the extension to {inseparable} composite problems where the regularizer couples coordinates \cite{CN23}. Our analysis adopts a similar proof skeleton; however, the coordinate updates must be reformulated as pushforward perturbations, requiring the re-development of descent arguments within the non-linear Wasserstein geometry. Furthermore, we make the dependence on coordinate-wise conditioning explicit through optimized non-uniform sampling and step-size choices. This level of detail, to our knowledge, has not been fully explored even in the inseparable Euclidean proximal{-gradient} literature.

\item \textbf{Coordinate and Block Methods in Wasserstein Space:} 
The most closely related work is \cite{YCY24}, which develops random-scan coordinate {proximal-gradient} schemes under structural assumptions that render subproblems tractable (e.g., via marginal-separable terms). Our work differs in three critical aspects: (i) our ``coordinates'' refer to axis-restricted transport directions for a single measure $\mu \in \mathcal{P}_2(\mathbb{R}^d)$, rather than block variables in a product space; (ii) we accommodate objectives that do not decouple across marginals, aligning with the inseparable perspective of \cite{CN23}; and (iii) we provide convergence rates with explicit anisotropic constants, highlighting the benefits of non-uniform sampling. Relatedly, coordinate ascent variational inference (CAVI) performs block-coordinate updates over product-form measures; recent convergence analyses for CAVI under log-concavity include \cite{AL24} (cyclic) and \cite{LZ24} (random-scan). Block majorization-minimization for generic Riemannian metric is also studied in~\cite{LBNL26}.

\end{itemize}

\section{Preliminaries}\label{sec:prelim}

This section gathers the essential background required for the remainder of the paper. Our results rest on two primary pillars: the classical algorithmic strategies for RCD and RCP in Euclidean space and the geometric framework of the Wasserstein-2 space $\mathcal{P}_2$. These are reviewed in Section~\ref{sec:euclid_rcd_rcp} and Section~\ref{sec:w2_prelim} respectively. Readers familiar with these foundational concepts may skim these subsections.

\subsection{Review of RCD and RCP in $\mathbb{R}^d$}\label{sec:euclid_rcd_rcp}
We briefly review the methods, assumptions, and convergence strategies for standard coordinate descent in $\mathbb{R}^d$. Throughout, let $U_i \in \mathbb{R}^{d \times d}$ denote the diagonal projection matrix onto the $i$-th coordinate, such that $U_i = e_i e_i^\top$, where $e_i$ is the $i$-th standard basis vector.

In Euclidean optimization, coordinate-based methods~\cite{N12, W15} are a natural choice for high-dimensional problems. The central theme is to update randomly selected coordinates, requiring only partial derivatives rather than full gradients. This significantly reduces the per-iteration computational overhead. Two prominent approaches are:

\begin{itemize}
    \item \textbf{Random Coordinate Descent (RCD):} Applicable when the gradient $\nabla f$ is accessible. At each iteration $k$, a coordinate $i_k \in \{1, \dots, d\}$ is sampled, and the iterate is updated as:
    \begin{equation}\label{eq:RCD}
    x_{k+1} = x_k - \gamma_{i_k} U_{i_k}\nabla f(x_k) \,,
    \end{equation}
    where $\gamma_{i_k}$ is a coordinate-specific step size. Computationally, this requires evaluating only the partial derivative $\partial_{i_k}f$, updating solely the $i_k$-th component of $x_k$.

    \item \textbf{Random Coordinate Proximal{-Gradient} (RCP):} Preferred for composite objectives of the form $f = g + \psi$, where $\psi$ may be non-smooth or lack a tractable gradient. At iteration $k$, a coordinate $i_k$ is sampled and the iterate is given by:
    \begin{equation}\label{eq:euclid_rcp}
        x_{k+1} = x_k + s_k e_{i_k} \,,
    \end{equation}
    where $s_k$ is the solution to the one-dimensional proximal subproblem:
    \begin{equation}\label{eq:euclid_rcp_subprob}
    s_k \in \arg\min_{s \in \mathbb{R}} \left\{ \partial_{i_k} g(x_k)s + \frac{\eta_{i_k}}{2}s^2 + \psi(x_k + s e_{i_k}) \right\} \,.
    \end{equation}
    Here, $\eta_{i_k}$ acts as a coordinate-wise proximity parameter. Similar to RCD, this approach restricts the update to the $i_k$-th directional information.
\end{itemize}

Since these methods rely on coordinate derivatives, the per-iteration cost is reduced by a factor of $d$ (assuming the evaluation of $\nabla f$ is $d$ times more expensive than a single $\partial_i f$). While the restricted information typically increases the total number of iterations required for a given accuracy $\epsilon$, coordinate methods offer the flexibility to adapt to the objective function's local geometry. In particular, for highly anisotropic or ill-conditioned functions, choosing coordinates and step sizes based on coordinate-wise smoothness can significantly reduce the total complexity (total iterations $\times$ cost per iteration).

To analyze the convergence, we adopt the following standard assumptions:
\begin{definition}[Coordinate-wise smoothness]\label{def:coord_smooth_euclid}
Let $f:\mathbb{R}^d \to \mathbb{R}$ be differentiable. We say $f$ is \emph{$L_i$-smooth along the $i$-th coordinate} if for all $x \in \mathbb{R}^d$ and all $s \in \mathbb{R}$,
\[
|\partial_i f(x + s e_i) - \partial_i f(x)| = \|U_i(\nabla f(x + s e_i) - \nabla f(x))\|_2 \leq L_i |s|\,.
\]
We denote the sum of these coordinate-wise constants as $L_{\rm sum} = \sum_{i=1}^d L_i$.
\end{definition}

\begin{definition}[Convexity, strong convexity, and PL]\label{def:euclid_convex_pl}
We say that $f$:
\begin{itemize}
    \item is \textbf{convex} if for all $x, y \in \mathbb{R}^d$:
    \[
    f(y) \geq f(x) + \langle \nabla f(x), y - x \rangle\,.
    \]
    \item is \textbf{$m$-strongly convex} if for all $x, y \in \mathbb{R}^d$:
    \[
    f(y) \geq f(x) + \langle \nabla f(x), y - x \rangle + \frac{m}{2} \|y - x\|_2^2\,.
    \]
    \item satisfies the \textbf{$m$-Polyak--{\L}ojasiewicz (PL) condition} if there exists $m > 0$ such that for all $x \in \mathbb{R}^d$:
    \[
    f(x) - f_* \leq \frac{1}{2m} \|\nabla f(x)\|_2^2\,,
    \]
    where $f_*$ denotes the minimum value of $f$.
\end{itemize}
\end{definition}

For our convergence analysis, we sometimes rely on the following standard assumption regarding the objective's landscape relative to its global minimum.

\begin{assumption}[Bounded sublevel set in $\mathbb{R}^d$]\label{asp:euclid_bls}
Let $X_*$ be the set of minimizers of $f(x)$ in $\mathbb{R}^d$. The function $f$ is assumed to have a bounded sublevel set at the initialization $x_0$; specifically, there exists a constant $R_0 \in (0, \infty)$ such that
\[
\sup_{x: \, f(x) \leq f(x_0)} \left( \inf_{x_* \in X_*} \|x - x_*\|_2 \right) \leq R_0\,.
\]
\end{assumption}

The complexity results for these methods are summarized below across different regularity regimes.

\begin{itemize}
    \item \textbf{RCD:} Designed to solve $\min_{x \in \mathbb{R}^d} f(x)$ for differentiable $f$. By updating a single randomly sampled coordinate $i_k$ per iteration according to \eqref{eq:RCD}, the method avoids full gradient computations. Let $p_i$ be the probability of selecting coordinate $i$. A standard choice is $p_i \propto L_i$ (the coordinate-wise smoothness parameter), with step sizes $\gamma_i \propto 1/L_i$. Under these settings, the iteration complexity $N$ is given by~\cite{N12}:
    \begin{description}
        \item[Non-convex:] To find an $\epsilon$-stationary point ($\min_{k \in [N]} \| \nabla f(x_k)\|^2 < \epsilon$): $N \geq \mathcal{O}(L_{\rm sum}/\epsilon)$.
        \item[Convex:] Under Assumption~\ref{asp:euclid_bls}, to achieve an optimality gap $f(x_N) - f_* \leq \epsilon$: $N \geq \mathcal{O}(L_{\rm sum}/\epsilon)$.
        \item[$m$-strongly convex / $m$-PL:] To achieve an optimality gap $\epsilon$: $N \geq \mathcal{O}(\frac{L_{\rm sum}}{m} \ln(1/\epsilon))$.
    \end{description}
    In all cases, RCD improves the computational cost relative to standard GD by replacing the factor $dL$ with $L_{\rm sum}$. Assuming a full gradient evaluation is $d$ times costlier than a partial derivative, this constitutes a significant speedup in anisotropic settings where $L_{\rm sum} \ll dL$.

    \item \textbf{RCP:} Solves $\min_x \{g(x) + \psi(x)\}$, where $\nabla g$ is accessible and $\psi$ is proximable (but potentially \emph{inseparable}). At each iteration, a coordinate is selected to perform a gradient step on $g$ while treating $\psi$ proximally, as defined in \eqref{eq:euclid_rcp}--\eqref{eq:euclid_rcp_subprob}. 
    Suppose $g$ and $\psi$ are $L$-smooth and $H$-smooth respectively, and along the $i$-th coordinate, $L_i$-smooth and $H_i$-smooth, respectively. Let $p_i$ be the probability of selecting coordinate $i$, and $\eta_i$ be the proximal parameters in~\eqref{eq:euclid_rcp_subprob}. We set them using:
    \begin{equation}\label{eqn:choice_grad-prox}
    \eta_i := L_i + \sqrt{H_i^2 + L_i^2}, \quad p_i := \frac{\eta_i}{\sum_{j=1}^d \eta_j} \propto L_i + \sqrt{H_i^2 + L_i^2} \,,
    \end{equation}
    we obtain:
    \begin{description}
        \item[Non-convex:] To find an $\epsilon$-stationary point: $N \geq \mathcal{O}(\mathcal{C}/\epsilon)$.
        \item[Convex:] To achieve an optimality gap $\epsilon$: $N \geq \mathcal{O}(\mathcal{C}/\epsilon)$.
        \item[$m$-strongly convex / $m$-PL:] To achieve an optimality gap $\epsilon$: $N \geq \mathcal{O}(\frac{\mathcal{C}}{m} \ln(1/\epsilon))$,
    \end{description}
    where
    \begin{equation}\label{eqn:calC}
    \mathcal{C} = 8\sum_{j=1}^d\eta_i=8 \sum_{j=1}^d (L_j + \sqrt{L_j^2 + H_j^2})\,.
    \end{equation}
    These results generalize those in~\cite{CN23} (which studied uniform $p_i$) to the coordinate-dependent sampling regime. We formalize these Euclidean results in Theorem~\ref{thm:rcp_conv_euclid} and provide proofs in Appendix~\ref{sec:proof_rcp_euclidean}.

    Similar to RCD, RCP also offers advantages over full proximal-gradient methods. While standard proximal-gradient methods~\cite{KNS16, PB14} scale with $d \cdot \max\{L, H\}$, RCP scales with the sum of coordinate-wise constants. Since $\sum_{j=1}^d (L_j + \sqrt{L_j^2 + H_j^2}) \leq 3d \cdot \max\{L, H\}$, RCP matches the full method even in the worst case and significantly outperforms it on highly anisotropic objectives where $\sum L_j \ll dL$ and $\sum H_j \ll dH$.
\end{itemize}

{Our objective is to extend these RCD and RCP ideas to the space of probability measures, where coordinate updates are reformulated through transport maps in Wasserstein geometry, while retaining the computational savings of coordinate methods.}

\medskip

\subsection{Probability Measure Space and the Wasserstein Metric}\label{sec:w2_prelim}

We now transition to the probability measure setting. Throughout this work, we operate on $\mathcal{P}_2(\mathbb{R}^d)$, the space of Borel probability measures with finite second moments. For any $\mu \in \mathcal{P}_2(\mathbb{R}^d)$, we define the inner product between two vector fields $f, g: \mathbb{R}^d \to \mathbb{R}^d$ as $\langle f, g \rangle_{\mu} := \int_{\mathbb{R}^d} f(x)^\top g(x) \, d\mu(x)$. The associated Hilbert space $\mathcal{L}^2(\mu; \mathbb{R}^d)$ consists of all measurable functions $f$ such that $\| f \|_{\mu}^2 := \langle f, f \rangle_{\mu} < \infty$.

To describe the transformation of measures, we utilize the pushforward operator:

\begin{definition}[Pushforward Measure]
Let $\mu \in \mathcal{P}_2(\mathbb{R}^d)$ and $T : \mathbb{R}^d \to \mathbb{R}^d$ be a measurable map. The pushforward of $\mu$ by $T$, denoted $T_{\#} \mu$, is the measure defined by
\[
(T_{\#} \mu)(A) := \mu(T^{-1}(A)) \quad \text{for every Borel set } A \subseteq \mathbb{R}^d.
\]
Equivalently, if a random variable $X$ has law $\mu$ ($X \sim \mu$), then $T(X)$ has law $T_{\#} \mu$.
\end{definition}

The geometry of $\mathcal{P}_2(\mathbb{R}^d)$ is characterized by the Wasserstein-2 metric:

\begin{definition}[Wasserstein-2 Distance and Optimal Coupling]
For any $\mu, \nu \in \mathcal{P}_2(\mathbb{R}^d)$, the Wasserstein-2 distance is defined as
\begin{equation}
\mathbb{W}_2(\mu, \nu) := \inf_{\pi \in \Pi(\mu, \nu)} \left( \int_{\mathbb{R}^d \times \mathbb{R}^d} \| x - y \|^2 \, d\pi(x,y) \right)^{1/2} , 
\end{equation}
where $\Pi(\mu, \nu)$ is the set of all couplings (joint distributions) with marginals $\mu$ and $\nu$. We denote by $\Pi_o(\mu, \nu)$ the set of optimal couplings where the infimum is attained.
\end{definition}

By endowing $\mathcal{P}_2(\mathbb{R}^d)$ with this metric, we can define a differential structure. The gradient of a functional $E$ in this geometry is defined through its action on the velocity fields of evolving measures:

\begin{definition}[Wasserstein Gradient]\label{def:wass_grad}
Let $E: \mathcal{P}_2(\mathbb{R}^d) \to \mathbb{R}$ be Wasserstein differentiable. The \emph{Wasserstein gradient} at $\mu$, denoted $\nabla_{\mathbb{W}}E[\mu]$, is the unique element in the set of $\overline{\{\nabla \phi : \phi \in C_c^\infty(\mathbb{R}^d)\}}^{\mathcal{L}^2(\mu)}$ such that for any absolutely continuous curve $(\mu_t)_{t \in [0, T)}$ with $\mu_0 = \mu$ and compatible velocity field $v_0$, we have:
\[
\left.\frac{d}{dt} E[\mu_t] \right|_{t=0} = \langle \nabla_{\mathbb{W}}E[\mu], v_0 \rangle_{\mu}\,.
\]
\end{definition}

As shown in~\cite{AGS08}, when $E$ is sufficiently regular, the Wasserstein gradient can be expressed via the standard Euclidean gradient of its first variation: $\nabla_{\mathbb{W}} E[\mu](x) = \nabla \left( \frac{\delta E}{\delta \mu}(x) \right)$.

This differential structure naturally allows for the formulation of \emph{Wasserstein Gradient Descent} (WGD). By analogy with standard Gradient Descent in Euclidean space, where iterates are updated via the gradient $\nabla f$, a functional $E[\mu]$ can be optimized by pushing forward the measure along the negative Wasserstein gradient:
\begin{equation}\label{eqn:def_gd}
    x_{k+1} = x_k - \gamma \nabla f(x_k) \quad \implies \quad \mu_{k+1} = \left( \operatorname{Id} - \gamma \nabla_{\mathbb{W}} E[\mu_k] \right)_{\#} \mu_k \, .
\end{equation}

To replicate the Euclidean coordinate-descent analysis, we require a notion of regularity that characterizes how the functional $E$ varies under axis-restricted transport perturbations, mirroring Definition~\ref{def:coord_smooth_euclid}.

\begin{definition}[Smoothness and coordinate-wise smoothness in Wasserstein space]\label{def:coord_wise_smooth}
Fix $\mu\in\mathcal P_2(\mathbb R^d)$. Let $E : \mathcal P_2(\mathbb R^d) \to \mathbb R$ be a functional. We say
\begin{itemize}
    \item $E$ is \emph{$L$-smooth at $\mu$} if for any vector field $T\in L^2(\mu;\mathbb R^d)$, the Wasserstein gradients satisfy the Lipschitz bound
\begin{equation}\label{eqn:smooth_pullback}
\Big\| \big(\nabla_{\mathbb W} E[\mu_T]\circ(\operatorname{Id} + T)-\nabla_{\mathbb W} E[\mu]\big)\Big\|_{\mu}
\leq L \,\|T\|_{\mu}\,,
\end{equation}
where $\mu_T := (\operatorname{Id} + T)_\#\mu$.
\item For a fixed $i \in \{1, \dots, d\}$, $E$ is \emph{$L_i$-smooth along the $i$-th coordinate at $\mu$} if, for any vector field $T \in \mathcal{L}^2(\mu; \mathbb{R}^d)$, the Wasserstein gradients satisfy the Lipschitz bound:
\begin{equation}\label{eqn:coord_smooth_pullback}
\Big\| U_i \big( \nabla_{\mathbb{W}} E[\mu_T] \circ (\operatorname{Id} + U_i T) - \nabla_{\mathbb{W}} E[\mu] \big) \Big\|_{\mu} \leq L_i \| U_i T \|_{\mu},
\end{equation}
where $\mu_T := (\operatorname{Id} + U_i T)_{\#} \mu$.
\end{itemize}
Consistent with the Euclidean setting, we define the aggregate smoothness constant as $L_{\rm sum} = \sum_{i=1}^d L_i$.
\end{definition}

Roughly speaking, this definition gives a Lipschitz regularity of the first order derivative, and can be considered as an upper bound of the Hessian. In $\mathbb{R}^d$, a bounded Hessian can be used to control the Taylor expansion of a function up to its second order. The same estimate can be obtained here:
\begin{proposition}[Coordinate-wise descent]\label{prop:coord_wise_smooth_1st_order}
Let $E : \mathcal P_2(\mathbb R^d) \to \mathbb R$ be differentiable.
Fix $\mu \in \mathcal P_2(\mathbb R^d)$ and $i \in [d]$, $E$ has the directional regularity as in Definition~\ref{def:coord_wise_smooth}, then for any vector field $T \in L^2(\mu; \mathbb R^d)$, we have:
\begin{equation}\label{eqn:coord_descent}
E[(\operatorname{Id}+U_iT)_\#\mu]
\leq E[\mu]
+ \big\langle \nabla_{\mathbb W} E[\mu],\,U_i T\big\rangle_{\mu}
+ \frac{L_i}{2}\,\|U_iT\|_{\mu}^2.
\end{equation}
\end{proposition}
The proof of Proposition~\ref{prop:coord_wise_smooth_1st_order} is deferred to Appendix~\ref{sec:lem:coord_wise_smooth_1st_order_proof}.

Many common energy functionals satisfy these coordinate-wise smoothness. We list a few examples here:

\begin{proposition}[Examples of Coordinate-wise Smoothness]\label{prop:examples_coord_smooth}
Examples of energy functionals that are smooth:
\begin{enumerate}[label=(\alph*)]
    \item \textbf{(Potential Energy):} Let $E[\mu] = \int_{\mathbb{R}^d} V(x) \, d\mu(x)$, where $V: \mathbb{R}^d \to \mathbb{R}$ is $L$-smooth and $L_i$-smooth along the $i$-th coordinate. Then $E$ is $L$-smooth and $L_i$-smooth along coordinate $i$ for all $\mu \in \mathcal{P}_2(\mathbb{R}^d)$.
    \item \textbf{(Function of the Mean):} Let $m(\mu) := \int_{\mathbb{R}^d} x \, d\mu(x)$ and $E[\mu] = \phi(m(\mu))$, where $\phi: \mathbb{R}^d \to \mathbb{R}$ is $L$-smooth and $L_i$-smooth along the $i$-th coordinate. Then $E$ is $L$-smooth and $L_i$-smooth along coordinate $i$ for all $\mu \in \mathcal{P}_2(\mathbb{R}^d)$.
    \item \textbf{(Interaction Energy):} Let $E[\mu] = \frac{1}{4} \iint_{\mathbb{R}^d \times \mathbb{R}^d} W(x - y) \, d\mu(x) d\mu(y)$, where $W: \mathbb{R}^d \to \mathbb{R}$ is $L$-smooth and $L_i$-smooth along the $i$-th coordinate. Then $E$ is \emph{$L$-smooth} and $L_i$-smooth along coordinate $i$ for all $\mu \in \mathcal{P}_2(\mathbb{R}^d)$.
    \item \textbf{(MMD-type Kernel Energy):} Let $\nu \in \mathcal{P}_2(\mathbb{R}^d)$ be fixed and 
    \begin{equation}\label{eqn:MMD_def}
    E[\mu] = \frac{1}{2} \iint_{\mathbb{R}^d \times \mathbb{R}^d} k(x - y) \, d(\mu - \nu)(x) d(\mu - \nu)(y),
    \end{equation}
    where $k: \mathbb{R}^d \to \mathbb{R}$ is $L$-smooth and $L_i$-smooth along the $i$-th coordinate. Then $E$ is \emph{$4L$-smooth} and $4L_i$-smooth along coordinate $i$ for all $\mu \in \mathcal{P}_2(\mathbb{R}^d)$.
\end{enumerate}
\end{proposition}

The proof of Proposition~\ref{prop:examples_coord_smooth} is provided in Appendix~\ref{sec:proof_prop:examples_coord_smooth}.

We now focus on applying these geometric concepts to the optimization of functionals. Just as the landscape of $f$ determines convergence behavior in $\mathbb{R}^d$, we require notions of geodesic convexity, the PL condition, and (for the non-strongly convex case) a bounded sublevel set assumption in $\mathcal{P}_2$.

\begin{definition}[Geodesic Convexity]\label{def:geod_convexity}
A Wasserstein differentiable functional $E : \mathcal{P}_2(\mathbb{R}^d) \to \mathbb{R}$ is \emph{$m$-geodesically convex} if for all $\mu, \nu \in \mathcal{P}_2(\mathbb{R}^d)$, and for any optimal coupling $\pi_o \in \Pi_o(\mu, \nu)$, we have:
\begin{equation}
E[\nu] \geq E[\mu] + \int_{\mathbb{R}^d \times \mathbb{R}^d} \nabla_{\mathbb{W}} E[\mu](x) \cdot (y - x) \, d\pi_o(x,y) + \frac{m}{2} \mathbb{W}_2^2(\mu, \nu) \, .
\end{equation}
When $m > 0$, we say $E$ is \emph{$m$-geodesically strongly convex}.
\end{definition}

This definition is a direct lifting of Euclidean convexity to the Wasserstein manifold. A useful relaxation of strong convexity is the Polyak-{\L}ojasiewicz (PL) condition, which allows for global convergence without requiring strict convexity.

\begin{definition}[Polyak-{\L}ojasiewicz Condition]\label{def:m_PL}
A differentiable functional $E : \mathcal{P}_2(\mathbb{R}^d) \to \mathbb{R}$ satisfies the \emph{$m$-Polyak-{\L}ojasiewicz ($m$-PL) condition} for some $m > 0$ if, for any $\mu \in \mathcal{P}_2(\mathbb{R}^d)$,
\begin{equation}
E[\mu] - E[\mu_*] \leq \frac{1}{2m} \| \nabla_{\mathbb{W}} E[\mu] \|_{\mu}^2 \, ,
\end{equation} 
where $\mu_*$ is a global minimizer of $E$.
\end{definition}

As in the Euclidean setting, the PL condition is weaker than strong convexity, a property that carries over to the Wasserstein space:
\begin{lemma}\label{lem:sc_implies_PL}
If $E : \mathcal{P}_2(\mathbb{R}^d) \to \mathbb{R}$ is differentiable and $m$-geodesically strongly convex, then it satisfies the $m$-PL condition.
\end{lemma}
The proof is provided in Appendix~\ref{sec:proof_lem:sc_implies_PL}. Finally, to analyze the convergence of RWCD in the case of $0$-geodesic convexity, we introduce a bounded sublevel set assumption, mirroring Assumption~\ref{asp:euclid_bls}.

\begin{assumption}[Bounded Sublevel Sets in $\mathcal{P}_2$]\label{asp:bounded_level_set}
Let $S$ be the set of minimizers of $E$ in $\mathcal{P}_2(\mathbb{R}^d)$. We assume $E$ has a bounded sublevel set at the initialization $\mu_0$; specifically, there exists $R_0 \in (0,\infty)$ such that:
\begin{equation}\label{eq:bounded_level_set}
\sup_{\mu : E[\mu] \leq E[\mu_0]} \left( \inf_{\mu_* \in S} \mathbb{W}_2(\mu, \mu_*) \right) \leq R_0\,.
\end{equation}
\end{assumption}

\section{Random Wasserstein Coordinate Descent (RWCD)}\label{sec:rcd}
In this section, we introduce the RWCD algorithm and establish its convergence properties across various landscape geometries.

To optimize the functional $E[\mu]$, we adopt the randomized coordinate-wise update strategy from Euclidean space, adapted to the geometry of $\mathcal{P}_2$. Specifically, at each iteration $k$, we update the measure $\mu_k$ by applying a pushforward map restricted to a randomly sampled coordinate $i_k$. Given a probability distribution $p = (p_1, \dots, p_d)$, we sample $i_k \sim p$ and update:
\begin{equation}\label{eqn:rcd_wasserstein}
\mu_{k+1} = (\operatorname{Id} - \gamma_{i_k} U_{i_k} \nabla_{\mathbb{W}}E[\mu_k])_{\#}\mu_k\,,
\end{equation}
where $\gamma_{i_k} > 0$ is the coordinate-specific step size. Throughout our analysis, we find that the optimal parameter configuration consists of setting:
\begin{equation}\label{eqn:choice_explicit}
p_i := \frac{L_i}{L_{\rm sum}} \, , \quad \gamma_i := \frac{1}{L_i} \, ,
\end{equation}
where $L_{\rm sum} = \sum_{i=1}^d L_i$\footnote{Note that this choice automatically guarantees the step {is an unbiased estimator of the gradient}, in the sense that $\mathbb{E}\left(\gamma_{i} U_{i} \nabla_{\mathbb{W}}E[\mu]\right) \propto
\nabla_{\mathbb{W}}E[\mu]$.}. 
This procedure is summarized in Algorithm~\ref{alg:rcd_wasserstein}.

\begin{algorithm}
\caption{Random Wasserstein Coordinate Descent (RWCD)}
\label{alg:rcd_wasserstein}
\begin{algorithmic}[1]
\State \textbf{Input:} Initial measure $\mu_0 \in \mathcal{P}_2(\mathbb{R}^d)$, coordinate smoothness $\{L_i\}_{i=1}^d$.
\State \textbf{Initialize:} Set $p_i = L_i/L_{\rm sum}$ and $\gamma_i = 1/L_i$.
\For{$k = 0, 1, \dots$}
    \State Sample coordinate $i_k \in \{1, \dots, d\}$ with $\mathbb{P}(i_k = i) = p_i$.
    \State Update the measure: $\mu_{k+1} = (\operatorname{Id} - \gamma_{i_k} U_{i_k} \nabla_{\mathbb{W}} E[\mu_k])_{\#} \mu_k$.
\EndFor
\end{algorithmic}
\end{algorithm}

The following theorem demonstrates that RWCD achieves the desired convergence rates, mirroring the efficiency of its Euclidean counterpart.

\begin{theorem}[Convergence of RWCD]\label{thm:rcd_conv}
Let $E : \mathcal{P}_2(\mathbb{R}^d) \to \mathbb{R}$ be Wasserstein differentiable and $L_i$-smooth along each coordinate $i \in \{1, \dots, d\}$. With $p_i$ and $\gamma_i$ chosen according to~\eqref{eqn:choice_explicit}, the iterates of Algorithm~\ref{alg:rcd_wasserstein} satisfy:
\begin{itemize}
    \item \textbf{Non-convex:} The energy is monotonically non-increasing ($E[\mu_{k+1}] \leq E[\mu_k]$). Furthermore,
    \[
    \mathbb{E}\left[\min_{0 \leq t < k} \|\nabla_{\mathbb{W}} E[\mu_t]\|_{\mu_t}^2\right] \leq \frac{2 L_{\rm sum} (E[\mu_0] - E[\mu_*])}{k}\,,
    \]
    achieving an $\epsilon$-approximate stationary point in $\mathcal{O}(L_{\rm sum}/\epsilon)$ iterations.
    
    \item \textbf{Geodesically Convex:} If $E$ is geodesically convex and satisfies Assumption~\ref{asp:bounded_level_set} with radius $R_0$, then:
    \[
    \mathbb{E}[E[\mu_k] - E[\mu_*]] \leq \frac{2 L_{\rm sum} R_0^2}{k}\,,
    \]
    achieving an optimality gap $\epsilon$ in $\mathcal{O}(L_{\rm sum}/\epsilon)$ iterations.
    
    \item \textbf{$m$-Strongly Convex / $m$-PL:} If $E$ satisfies the $m$-PL condition or is $m$-geodesically strongly convex, then:
    \[
    \mathbb{E}[ E[\mu_k] - E[\mu_*] ] \leq \left(1 - \frac{m}{L_{\rm sum}} \right)^k (E[\mu_0] - E[\mu_*])\,,
    \]
    achieving an optimality gap $\epsilon$ in $\mathcal{O}(\frac{L_{\rm sum}}{m} \ln\frac{1}{\epsilon})$ iterations.
\end{itemize}
\end{theorem}

Theorem~\ref{thm:rcd_conv} highlights that RWCD inherits the benefits of randomized coordinate updates: the convergence rates depend on the average smoothness $L_{\rm sum}$ rather than the global smoothness constant, providing significant acceleration for anisotropic functionals.

The proof of the theorem strongly relies on the following descent lemma:
\begin{lemma}[Expected One-Step Descent]\label{lem:RCD_one_step}
Let $E$ satisfies the same assumption as
in Theorem~\ref{thm:rcd_conv}, the iterates of Algorithm~\ref{alg:rcd_wasserstein} satisfy {$E[\mu_{k+1}] \leq E[\mu_k]$, and}:
\begin{equation}\label{eqn:RCD_one_step}\mathbb{E} [ E[\mu_{k+1}] \mid \mu_k ] \leq E[\mu_k] - \frac{1}{2 L_{\rm sum}} \| \nabla_{\mathbb{W}} E[\mu_k] \|_{\mu_k}^2.\end{equation}
\end{lemma}
Lemma~\ref{lem:RCD_one_step} establishes that the functional $E[\mu_k]$ decreases monotonically at each iteration, with the {expected} rate of descent determined by the norm of the full Wasserstein gradient. This result directly mirrors the coordinate-descent lemma in Euclidean space (see, e.g., \cite[Section 6.2]{WR22}). We defer the proof of this lemma to Appendix~\ref{sec:RWCD_proof} and proceed to complete the proof of the main theorem.
\begin{proof}[Proof of Theorem~\ref{thm:rcd_conv}]
We analyze the convergence based on the landscape of $E$:
\begin{itemize}
\item \textbf{Non-convex case:} Rearranging the descent bound \eqref{eqn:RCD_one_step}, we have:
$$\left\Vert \nabla_{\mathbb{W}} E[\mu_k] \right\Vert_{\mu_k}^2 \leq 2 L_{\mathrm{sum}} \left( E[\mu_k] - \mathbb{E}\left[ E[\mu_{k+1}] \left\vert\right. \mu_k \right] \right) \,.$$
Summing across $k$ iterations and taking the total expectation:
\[
\begin{aligned}
\mathbb{E}\left[\min_{0 \leq t < k} \| \nabla_{\mathbb{W}} E[\mu_t] \|_{\mu_t}^2\right] 
&\leq \mathbb{E}\left[\frac{1}{k} \sum_{t=0}^{k-1} \| \nabla_{\mathbb{W}} E[\mu_t] \|_{\mu_t}^2\right] \\
&\leq \frac{2 L_{\rm sum}}{k} \sum_{t=0}^{k-1} \mathbb{E} \big[ E[\mu_t] - E[\mu_{t+1}] \big] \\
&\leq \frac{2 L_{\rm sum}}{k} ( E[\mu_0] - E[\mu_*] ) \, .
\end{aligned}
\]
\item \textbf{Geodesically convex case ($m=0$):} From Definition~\ref{def:geod_convexity}, for any optimal coupling $\pi_o \in \Pi_o(\mu_*, \mu_k)$, we have:
\begin{equation}\label{eqn:convex_grad_bd}
\begin{aligned}
E[\mu_k] - E[\mu_*] &\leq \int_{\mathbb{R}^d \times \mathbb{R}^d} \langle \nabla_{\mathbb{W}} E[\mu_k](y), y - x \rangle \, d \pi_o(x,y) \\
&\leq \sqrt{\int \| \nabla_{\mathbb{W}} E[\mu_k](y) \|^2 d \pi_o} \cdot \sqrt{\int \| y - x \|^2 d \pi_o} \\ 
&= \| \nabla_{\mathbb{W}} E[\mu_k] \|_{\mu_k} \cdot \mathbb{W}_2(\mu_k, \mu_*) \leq R_0 \| \nabla_{\mathbb{W}} E[\mu_k] \|_{\mu_k} \, ,
\end{aligned}
\end{equation}
where we applied Cauchy-Schwarz and Assumption~\ref{asp:bounded_level_set}. Conditioned on $\mu_k$ and utilizing \eqref{eqn:RCD_one_step}:
\[
\begin{aligned}
\mathbb{E}[ E[\mu_{k+1}] - E[\mu_*] \mid \mu_k ] &\leq E[\mu_k] - E[\mu_*] - \frac{1}{2 L_{\rm sum}} \| \nabla_{\mathbb{W}} E[\mu_k] \|_{\mu_k}^2 \\
&\leq E[\mu_k] - E[\mu_*] - \frac{(E[\mu_k] - E[\mu_*])^2}{2 L_{\rm sum} R_0^2} \, .
\end{aligned}
\]
Let $\phi_k := \mathbb{E}[E[\mu_k] - E[\mu_*]]$. By Jensen's inequality, this recurrence implies $\phi_{k+1} \leq \phi_k - \frac{\phi_k^2}{2 L_{\rm sum} R_0^2}$. Consequently:
\begin{equation}\label{eq:phi_sequel}
\frac{1}{\phi_{k+1}} - \frac{1}{\phi_k} = \frac{\phi_k - \phi_{k+1}}{\phi_k \phi_{k+1}} \geq \frac{\phi_k - \phi_{k+1}}{\phi_k^2} \geq \frac{1}{2 L_{\rm sum} R_0^2} \, .
\end{equation}
Applying this recursively yields $\frac{1}{\phi_k} \geq \frac{k}{2 L_{\rm sum} R_0^2}$, which concludes the proof for the convex case.
\item \textbf{$m$-strongly convex or $m$-PL case:} By \eqref{eqn:RCD_one_step} and the $m$-PL property:
\[
\mathbb{E}[ E[\mu_{k+1}] \mid \mu_k ] - E[\mu_k] \leq -\frac{1}{2 L_{\rm sum}} \| \nabla_{\mathbb{W}} E[\mu_k] \|_{\mu_k}^2 \leq -\frac{m}{L_{\rm sum}} ( E[\mu_k] - E[\mu_*] ) \, .
\]
Subtracting $E[\mu_*]$ and rearranging:
\[
\mathbb{E}[ E[\mu_{k+1}] - E[\mu_*] \mid \mu_k ] \leq \left( 1 - \frac{m}{L_{\rm sum}} \right) ( E[\mu_k] - E[\mu_*] ) \, .
\]
We conclude by induction.
\end{itemize}
\end{proof}
It is important to note that though the general proof strategy mirrors that for Euclidean RCD method, the Wasserstein geometry plays a crucial role, as is evident in examples such as~\eqref{eqn:convex_grad_bd}.

\section{Random Wasserstein Coordinate Proximal{-Gradient} (RWCP)}\label{sec:rcp}

This section is dedicated to the design of RWCP, extending the Random Coordinate Proximal{-Gradient} (RCP) framework to the Wasserstein setting. We consider optimization problems where the objective functional has a composite structure:
\begin{equation}\label{eqn:composite_obj}
\min_{\mu \in \mathcal{P}_2(\mathbb{R}^d)} E[\mu] := G[\mu] + \Psi[\mu] \,,
\end{equation}
where $G$ and $\Psi$ are Wasserstein differentiable, and we assume that while $\nabla_{\mathbb{W}}G$ is accessible, $\nabla_{\mathbb{W}}\Psi$ may not be explicitly available or is handled more efficiently via a proximal-type subproblem.

Following the Euclidean strategy in \eqref{eq:euclid_rcp_subprob}, at each iteration $k$, we sample a coordinate $i_k \in \{1, \dots, d\}$ with probability $p_i$ and update the measure:
\begin{equation}\label{eqn:rcp_update_xi}
\mu_{k+1} := (\operatorname{Id} + T_k)_{\#} \mu_k,
\end{equation}
where $T_k$ is the solution to the coordinate-restricted variational subproblem:
\begin{equation}\label{eqn:var_diff_grad_prox}
T_k \in \arg\min_{T \in \mathcal{L}^2(\mu_k; \mathbb{R}^d)} \left\{ \langle \nabla_{\mathbb{W}} G[\mu_k], U_{i_k} T \rangle_{\mu_k} + \frac{\eta_{i_k}}{2} \| T \|_{\mu_k}^2 + \Psi [ (\operatorname{Id} + U_{i_k} T)_{\#} \mu_k ] \right\}.
\end{equation}
The structure of \eqref{eqn:var_diff_grad_prox} ensures the displacement field $T$ vanishes outside the sampled coordinate direction $i_k$. This is the coordinate-wise analogue of the standard {Wasserstein Proximal-Gradient (WPG)} variational subproblem:
\begin{equation}\label{eqn:var_WPG}
T_k \in \arg\min_{T \in \mathcal{L}^2(\mu_k; \mathbb{R}^d)} \left\{ \langle \nabla_{\mathbb{W}} G[\mu_k], T \rangle_{\mu_k} + \frac{\eta}{2} \| T \|_{\mu_k}^2 + \Psi [ (\operatorname{Id} + T)_{\#} \mu_k ] \right\}.
\end{equation}

Throughout this section, we use the same sampling probabilities $p_i$ and proximal parameters $\eta_i$ as in the Euclidean space as in~\eqref{eqn:choice_grad-prox}. The procedure is summarized in Algorithm~\ref{alg:rcp_wasserstein}.

\begin{algorithm}[htb]
\caption{Random Wasserstein Coordinate Proximal{-Gradient} (RWCP)}
\label{alg:rcp_wasserstein}
\begin{algorithmic}[1]
\State \textbf{Input:} Initial measure $\mu_0 \in \mathcal{P}_2(\mathbb{R}^d)$, parameters $\{L_i, H_i\}_{i=1}^d$.
\State \textbf{Set:} $\eta_i = L_i + \sqrt{H_i^2 + L_i^2}$ and $p_i = \eta_i / \sum_{j=1}^d \eta_j$.
\For{$k = 0, 1, 2, \dots$}
    \State Sample $i_k \in \{1, \dots, d\}$ such that $\mathbb{P}(i_k = i) = p_i$.
    \State Compute $T_k$ by solving the subproblem \eqref{eqn:var_diff_grad_prox}.
    \State Update the measure: $\mu_{k+1} \gets (\operatorname{Id} + T_k)_{\#} \mu_k$.
\EndFor
\end{algorithmic}
\end{algorithm}

The algorithm exhibits convergence properties that match those established in the Euclidean setting, stated in the following theorem.
\begin{theorem}[Convergence of RWCP]\label{thm:rcp_conv}
Let $E = G + \Psi$. Assume $G$ is $L_i$-smooth and $\Psi$ is $H_i$-smooth along each coordinate $i \in \{1, \dots, d\}$. Under the parameter choices in \eqref{eqn:choice_grad-prox}, and denote $\mathcal{C} := 8 \sum_{j=1}^d \eta_j$, as in~\eqref{eqn:calC}, the iterates of Algorithm~\ref{alg:rcp_wasserstein} converges as follows:
\begin{itemize}
    \item \textbf{Non-convex:} $E[\mu_{k+1}] \leq E[\mu_k]$. Furthermore, 
    \[ \mathbb{E}\left[\min_{0 \leq t < k} \|\nabla_{\mathbb{W}} E[\mu_t]\|_{\mu_t}^2\right] \leq \frac{\mathcal{C}(E[\mu_0] - E[\mu_*])}{k}. \]
    \item \textbf{Geodesically Convex:} If $E$ is geodesically convex and satisfies Assumption~\ref{asp:bounded_level_set} with radius $R_0$, then
    \[ \mathbb{E}[E[\mu_k] - E[\mu_*]] \leq \frac{\mathcal{C} R_0^2}{k}. \]
    \item \textbf{$m$-Strongly Convex / $m$-PL:} If $E$ is $m$-geodesically strongly convex or satisfies the $m$-PL condition, then
    \[ \mathbb{E}[E[\mu_k] - E[\mu_*]] \leq \left(1 - \frac{2m}{\mathcal{C}} \right)^k (E[\mu_0] - E[\mu_*]). \]
\end{itemize}
\end{theorem}

Similar to the RWCD case, this theorem is built upon the following lemma, whose proof is deferred to Appendix~\ref{sec:RWCP_proof}.
\begin{lemma}[Descent Lemma for RWCP]\label{lem:descent_lemma_grad-prox}
Assume $E$ satisfies the same assumption as in Theorem~\ref{thm:rcp_conv}, the iterates of Algorithm~\ref{alg:rcp_wasserstein} satisfy $E[\mu_{k+1}] \leq E[\mu_k]$, and:
\begin{equation}
\mathbb{E}\left[E[\mu_{k+1}] - E[\mu_k] \mid \mu_k \right] \leq - \frac{1}{\mathcal{C}} \| \nabla_{\mathbb{W}} E[\mu_k] \|_{\mu_k}^2 \,.
\end{equation}
\end{lemma}

\begin{proof}[Proof of Theorem~\ref{thm:rcp_conv}]
Observe that Lemma~\ref{lem:descent_lemma_grad-prox} provides a descent bound structurally identical to Lemma~\ref{lem:RCD_one_step}, with the constant $2L_{\rm sum}$ replaced by $\mathcal{C}$. Consequently, the proof for all three regimes follows the exact steps detailed in the proof of Theorem~\ref{thm:rcd_conv}, substituting the respective constants.
\end{proof}

\section{Numerical Experiments}\label{sec:num}

In this section, we report our numerical findings and validate the theoretical convergence rates established in the previous sections. The objective functionals we study include interaction-type, MMD-type, and a mean-field limit of a two-layer NN-training. They naturally appear in many machine learning tasks.
Since a probability distribution is inherently an infinite-dimensional object, numerical implementation requires a finite-dimensional representation. 
Among the available choices, we employ a particle representation, approximating the measure $\mu_k$ at each iteration by an empirical measure:
\begin{equation}\label{eqn:mu_ensemble}
\mu_{k} = \frac{1}{N} \sum_{n=1}^{N} \delta_{x_{k}^{(n)}} \, ,
\end{equation}
where $\{x_{k}^{(n)}\}_{n=1}^N \subset \mathbb{R}^d$ are the positions of $N$ interacting particles. 

Throughout this section, we demonstrate convergence by plotting error curves against computational costs measured in \emph{coordinate-gradient evaluations}. To ensure a fair comparison with full-gradient methods, one random coordinate update counts as one unit of cost, while one full gradient update (WGD) counts as $d$ units.

\subsection{Example 1: 2D Anisotropic Potential and Interaction}\label{sec:2d_motivating}

The first example considers the following objective functional:
\begin{equation}\label{eqn:2D_exp}
E[\mu] := \underbrace{\int_{\mathbb{R}^2} \frac{x^\top \mathbf{P} x}{2} \, d\mu(x)}_{E_{\mathrm{pot}}[\mu]} + \underbrace{\frac{1}{8} \iint_{\mathbb{R}^2 \times \mathbb{R}^2} (x-y)^\top \mathbf{Q} (x-y) \, d\mu(x) \, d\mu(y)}_{E_{\mathrm{int}}[\mu]} \, ,
\end{equation}
which is composed of a linear potential term $E_{\mathrm{pot}}[\mu]$ and a quadratic interaction term $E_{\mathrm{int}}[\mu]$. We set
\[
V(x) = \frac{1}{2} x^\top \mathbf{P} x, \quad W(z) = \frac{1}{2} z^\top \mathbf{Q} z, \quad \text{with} \quad 
\mathbf{P} = \mathbf{Q} = \begin{bmatrix}
1000 & \frac{11100}{1111} \\
\frac{11100}{1111} & 1
\end{bmatrix}.
\]
In this setting, the unique optimizer is a Dirac delta located at the origin $x_* = 0$, so $\min_{\mu \in \mathcal{P}_2(\mathbb{R}^d)} E[\mu] = E[\delta_{x_*}] = 0$.

We can a-priori analyze the smoothness constants. Since the Hessian $\nabla^2 V = \mathbf{P}$, $V$ is globally $\|\mathbf{P}\|_2$-smooth and coordinate-wise $\mathbf{P}_{ii}$-smooth. By applying Proposition~\ref{prop:examples_coord_smooth}(a), $E_{\mathrm{pot}}$ is $L^{\mathrm{pot}} = \|\mathbf{P}\|_2$-smooth and $L^{\mathrm{pot}}_i = \mathbf{P}_{ii}$-smooth along coordinate $i$. 
Similarly, for the interaction term, Proposition~\ref{prop:examples_coord_smooth}(c) imply $E_{\mathrm{int}}$ is $L^{\mathrm{int}} = \|\mathbf{Q}\|_2$-smooth and $L^{\mathrm{int}}_i = \mathbf{Q}_{ii}$-smooth. By additivity, the full objective $E$ is $L$-smooth with $L = \|\mathbf{P}\|_2 + \|\mathbf{Q}\|_2$ and $L_i$-smooth along coordinate $i$ with $L_i = \mathbf{P}_{ii} + \mathbf{Q}_{ii}$. For the given $\mathbf{P}$ and $\mathbf{Q}$, this yields $L_1 = 2000 \gg L_2 = 2$, and $L \approx 2000.1$, highlighting the extreme anisotropy.

\textbf{Experimental Setup:}
We initialize $\mu_0$ using $N=2000$ particles drawn i.i.d.\ from $\mathcal{N}(0, I_d)$. We compare WGD with a fixed step size $h = 1/L$ and RWCD with probabilities $p_i \propto L_i$ and step sizes $h_i = 1/L_i$. The Wasserstein gradient for~\eqref{eqn:2D_exp} is given by:
\[
\nabla_{\mathbb{W}} E[\mu](x) = \mathbf{P} x + \frac{1}{2} \int_{\mathbb{R}^d} \mathbf{Q}(x-y) \, d\mu(y).
\]
In the particle representation, the update for WGD is:
\begin{equation}\label{eqn:particle_wgd}
x_{k+1}^{(n)} = x_k^{(n)} - h \nabla_{\mathbb{W}} E\left[\frac{1}{N} \sum_{m=1}^N \delta_{x_k^{(m)}}\right](x_k^{(n)}) \, , \quad n=1, \dots, N \, ,
\end{equation}
and for RWCD:
\begin{equation}\label{eqn:particle_rwcd}
x_{k+1}^{(n)} = x_k^{(n)} - h_{i_k} U_{i_k} \nabla_{\mathbb{W}} E\left[\frac{1}{N} \sum_{m=1}^N \delta_{x_k^{(m)}}\right](x_k^{(n)}) \, , \quad n=1, \dots, N \, .
\end{equation}
We run 50 independent trials for RWCD and report the median and the 10th-90th percentile band.

\textbf{Numerical Results:}
Numerical findings are presented in Figures~\ref{fig:2D_exp} and \ref{fig:2D_cloud}. 
Against coordinate-gradient evaluations, RWCD converges significantly faster than WGD. On the semi-log plots, WGD displays a linear convergence, whereas RWCD exhibits a distinct staircase-like pattern. 
Since $p_1 > 0.999$, the first coordinate is sampled almost exclusively and reaches its relative equilibrium nearly instantaneously. Consequently, the long plateaus in the energy curve represent the periods where $x_1$ is repeatedly sampled but provides no further energy reduction. Visible drops in the energy occur only on the rare occasions ($p_2 < 0.001$) when the second coordinate is sampled, which shifts the system's state and allows the optimization to progress to a new level. Both methods show an almost vertical drop at the outset, which stems from the anisotropy of $E[\mu]$: the first coordinate has essentially vanished for both methods in the first few iterations, also see Figure~\ref{fig:2D_cloud} middle left panel.

Figure~\ref{fig:2D_cloud} illustrates the particle cloud evolution. It is evident that RWCD particles collapse toward the origin $x_*=0$ much faster than WGD.

\begin{figure}[htbp]
  \centering
  \includegraphics[width=1.0\textwidth]{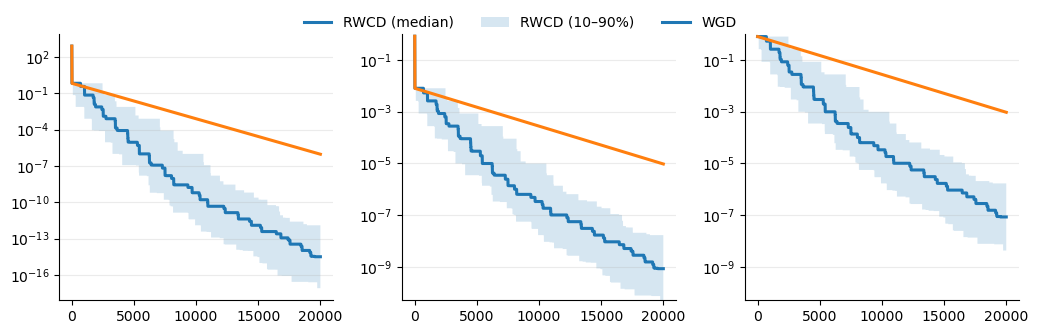}
  \caption{Example 1: RWCD vs.\ WGD on the 2D quadratic objective~\eqref{eqn:2D_exp} using $N=2000$ particles.
  Horizontal axis: work measured in coordinate-gradient evaluations (one WGD step equals $d=2$ work units).
  Blue: RWCD median with $10$--$90\%$ band across $50$ runs; orange: WGD. The three panels respectively show $E[\mu_k]$, empirical mean of the first coordinate of $N$ samples, empirical mean of the second coordinate of $N$ samples.}
  \label{fig:2D_exp}
\end{figure}

\begin{figure}[htbp]
  \centering
  \includegraphics[width=1.0\textwidth]{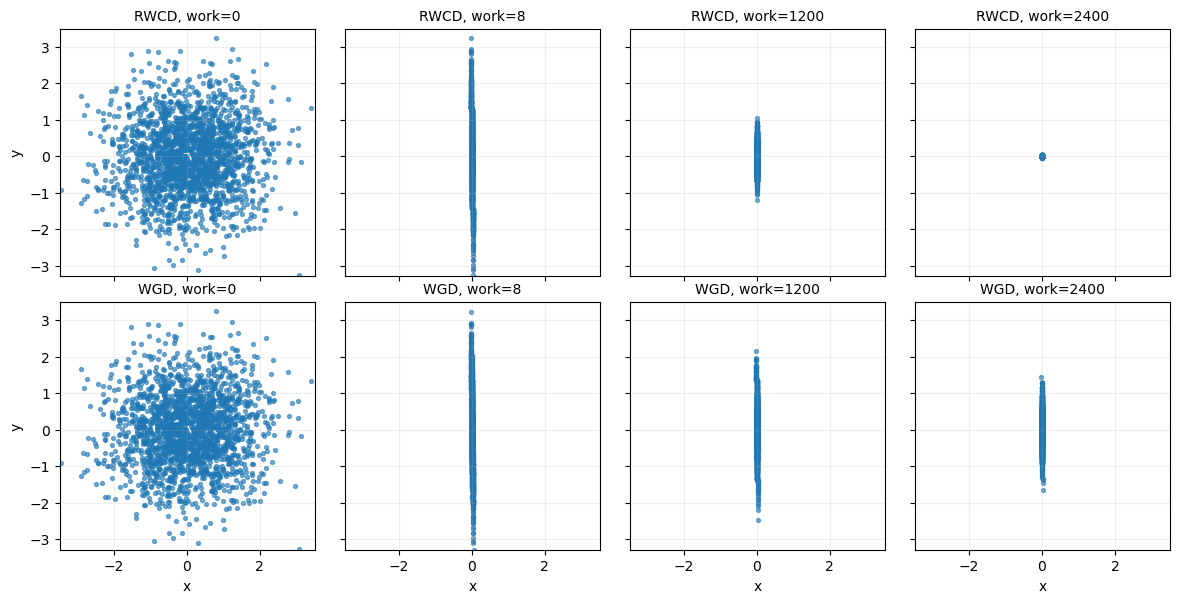}
  \caption{Example 1: Particle cloud snapshots for the 2D quadratic experiment~\eqref{eqn:2D_exp}.
  Top row: RWCD (a median trial); bottom row: WGD.
  Snapshots are taken at work $=0, 8, 1200, 2400$, where one WGD iteration is counted as $d=2$ work units.}
  \label{fig:2D_cloud}
\end{figure}

\subsection{High-dimensional Experiments}\label{sec:high_dim}

We now investigate the performance of the proposed solvers in higher-dimensional spaces, where the benefits of coordinate-wise importance sampling become more pronounced.

\subsubsection{Example 2: 50D, Potential and Interaction Functional}\label{subsubsec:highdim_quadratic}

In this example, we extend the setup of Example 1 to a higher-dimensional regime ($d=50$). The objective function remains the composite quadratic potential and interaction functional defined in \eqref{eqn:2D_exp}. 

To simulate a poorly conditioned landscape, we construct two positive definite matrices $\mathbf{P}, \mathbf{Q} \in \mathbb{R}^{d \times d}$.
Both have eigenvalues $\{\lambda_j\}_{j=1}^d$ are logarithmically equi-spaced between $1$ and $10^3$, but they have different eigenvectors.
This construction ensures a substantial variation in the coordinate-wise smoothness constants $\{L_i\}$.

Following our theory, we can compute the functional is $L$-smooth with $L = \|\mathbf{P}\|_2 + \|\mathbf{Q}\|_2 = 2000$, and $L_i$-smooth along each coordinate with $L_i = \mathbf{P}_{ii} + \mathbf{Q}_{ii}$. We set the number of particles to $N=2000$. For a fair comparison, we run WGD for $2000$ full iterations and RWCD for $2000d = 1\times 10^5$ coordinate updates, ensuring both methods use the same total number of coordinate-gradient evaluations.

Numerical results are displayed in Figure~\ref{fig:potential_interaction_explicit}. RWCD significantly outperforms WGD, with the energy decaying much faster per unit of computational work. This result is consistent with the conditioning advantage predicted by our theory: while WGD is restricted by the global smoothness constant (the largest eigenvalue), RWCD effectively adapts to the local geometry of each coordinate. 

\begin{figure}[htbp]
  \centering
  \includegraphics[width=0.8\textwidth]{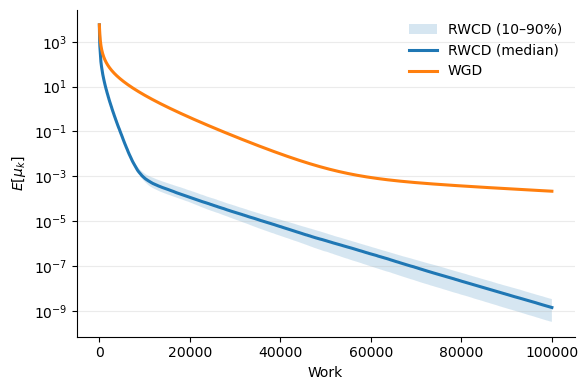}
  \caption{Example 2: RWCD vs.\ WGD for the quadratic potential and interaction objective with $N=2000$ and $d=50$. The horizontal axis represents work measured in coordinate-gradient evaluations. The blue curve shows the RWCD median with a $10$--$90\%$ confidence band across $50$ independent runs, while the orange curve represents WGD.}
  \label{fig:potential_interaction_explicit}
\end{figure}

\subsubsection{Example 3: 50D, MMD-type Functional}\label{subsubsec:highdim_mmd}

We next consider the MMD-type functional \eqref{eqn:MMD_def} with an anisotropic Gaussian kernel:
\begin{equation}
k(z) = \exp\left( -\frac{1}{2}\sum_{i=1}^d \lambda_i z_i^2 \right),
\end{equation}
where the coefficients $\{\lambda_i\}_{i=1}^d$ are logarithmically equi-spaced between $10^{-3}$ and $1$. 
By Proposition~\ref{prop:examples_coord_smooth}, this MMD functional is $L$-smooth with $L = 4\max_i \lambda_i$, and is coordinate-wise smooth with smoothness constants $L_i = 4\lambda_i$ for each $i \in \{1, \dots, d\}$.

\textbf{Experimental Setup:}
We set the dimension $d=50$. The target measure $\nu$ is an empirical distribution of $200$ particles sampled i.i.d.\ from $\mathcal{N}(0, I_d)$. The algorithm is initialized using $200$ particles sampled i.i.d.\ from $\mathcal{N}(m, \sigma^2 I_d)$, with $m = (-0.5, 0, \dots, 0)$ and $\sigma = 0.5$. Step sizes and sampling probabilities are chosen according to the derived smoothness coefficients, following the same importance sampling strategy used in the previous experiments.

\textbf{Numerical Results:}
Since this objective is non-convex, our theory characterizes the convergence of the minimum squared gradient norm $\min_{0\leq t < k} \|\nabla_{\mathbb{W}} E[\mu_t]\|_{\mu_t}^2$. We plot this convergence against computational work in Figure~\ref{fig:mmd_exp}. 

Similar to the previous examples, RWCD exhibits substantially faster decay than WGD. This result is consistent with the conditioning advantage suggested by our coordinate-wise smoothness analysis, even in a non-convex landscape. It is also noteworthy that in this experiment, the RWCD performance band is extremely narrow, indicating that the benefits of the coordinate-wise approach are highly consistent across independent runs.

\begin{figure}[htbp]
  \centering
  \includegraphics[width=0.8\textwidth]{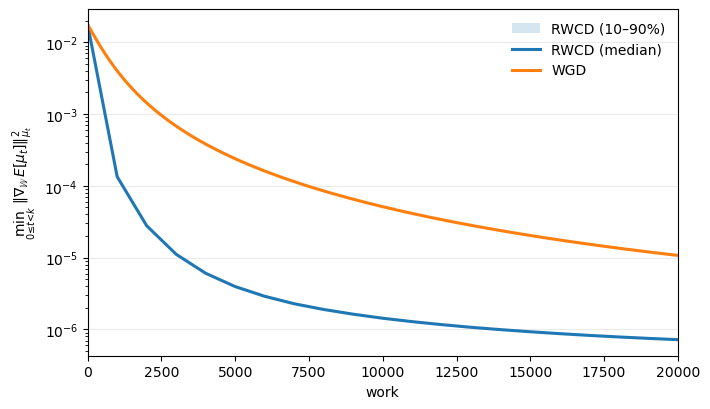}
  \caption{Example 3: RWCD vs.\ WGD for the MMD-type functional with an anisotropic Gaussian kernel in $d=50$. The horizontal axis represents work measured in coordinate-gradient evaluations. The blue curve denotes the RWCD median with a $10$--$90\%$ band across $50$ runs; the orange curve denotes WGD. Note that the RWCD band is visually indistinguishable from the median curve at this resolution.}
  \label{fig:mmd_exp}
\end{figure}

\subsection{Example 4: Composite Objective with Smooth Inseparable Regularizer}\label{subsec:highdim_composite}

In this experiment, we consider the composite functional $E = G + \Psi_\epsilon$, where the regularizer $\Psi_\epsilon$ is non-separable across coordinates. The interaction term is defined as:
\[
G[\mu] = \frac{1}{8} \iint_{\mathbb{R}^d \times \mathbb{R}^d} (x-y)^\top \mathbf{Q} (x-y) \, d\mu(x) \, d\mu(y), \quad \mathbf{Q} = \mathrm{diag}(q_1, \dots, q_d).
\]
Defining the kernel $W(z) := \frac{1}{2} z^\top \mathbf{Q} z$, we have $G[\mu] = \frac{1}{4} \iint W(x-y) \, d\mu(x)d\mu(y)$. The regularizer is a linear functional induced by the potential $V_\epsilon$:
\[
\Psi_\epsilon[\mu] = \int_{\mathbb{R}^d} V_\epsilon(x) \, d\mu(x) \, , \quad \text{with} \quad V_\epsilon(x) := \sum_{j=1}^d r_j \sqrt{(\mathbf{A} x)_j^2 + \epsilon^2} \, ,
\]
where $\epsilon = 10^{-2}$ and $\mathbf{A} \in \mathbb{R}^{d \times d}$ is an orthogonal matrix that mixes all coordinates. Since $\mathbf{Q}$ is positive definite and the unique minimizer of $V_\epsilon(x)$ is the origin, the optimizer of $E$ is $\mu_* = \delta_0$, with minimum energy $E_* = \epsilon \sum_{j=1}^d r_j$.

Following the discussion in Example 1, $G$ is $L$-smooth globally with $L = \|\mathbf{Q}\|_2$, and $L_i$-smooth along coordinate $i$ with $L_i = \mathbf{Q}_{ii}$. For the regularizer, the Hessian of $V_\epsilon$ is:
\[
\nabla^2 V_\epsilon(x) = \epsilon^2 \mathbf{A}^\top \operatorname{diag}_{j=1}^d \left( r_j \big( (\mathbf{A}x)_j^2 + \epsilon^2 \big)^{-3/2} \right) \mathbf{A}.
\]
Coordinate-wise, this implies $\partial_{ii} V_\epsilon(x) \leq \frac{1}{\epsilon} \sum_{j=1}^d r_j \mathbf{A}_{j,i}^2 =: H_i$. Globally, the operator norm is bounded by $\|\nabla^2 V_\epsilon(x)\|_2 \leq \epsilon^{-1} \max_j r_j =: H$. By Proposition~\ref{prop:examples_coord_smooth}(a), $E$ is $(L+H)$-smooth globally and $(L_i+H_i)$-smooth along coordinate $i$.

\textbf{Experimental Setup:}
We set $d=50$ and choose weights $\{r_j\}$ and $\mathbf{A}$ such that $\{H_i\}$ are approximately equi-spaced on a log scale between $1.5$ and $4959$ (see Figure~\ref{fig:rwcp_H_i} for the realized $H_i$). We initialize $N=200$ particles i.i.d.\ from $\mathcal{N}(m, \sigma^2 I_d)$ with $m = (1/8, \dots, 1/8)$ and $\sigma = 1/8$. We compare {{WPG}} ($2 \times 10^4$ iterations) against RWCP ($2 \times 10^4 d$ coordinate updates).

\begin{figure}[htbp]
  \centering
  \includegraphics[width=0.6\textwidth]{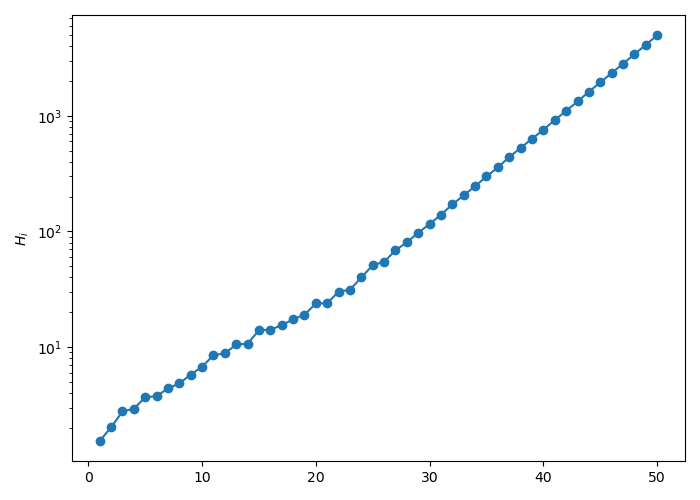}
  \caption{Example 4: Coordinate-wise smoothness constants $H_i$ of the regularizer $\Psi_\epsilon$. The $y$-axis shows $H_i$ on a logarithmic scale, and the $x$-axis indexes the coordinates after sorting by $H_i$.}
  \label{fig:rwcp_H_i}
\end{figure}

\textbf{Particle-wise Variational Subproblems:}
The variational subproblems~\eqref{eqn:var_diff_grad_prox} and~\eqref{eqn:var_WPG} are formulated directly over $\mu\in\mathcal{P}_2$. When $\mu$ is an empirical measure and takes on the form of~\eqref{eqn:mu_ensemble}, the update becomes solving the following variational subproblems for each particle $n=1,\dots,N$:
\begin{itemize}
    \item[--] {{WPG} Subproblem:} Solve for the full displacement vector $v \in \mathbb{R}^d$:
    \begin{equation}\label{eqn:particle_wpg_x}
    T_{k,n} \in \arg\min_{v \in \mathbb{R}^d} \left\{ v \cdot \nabla \frac{\delta G}{\delta \mu}[\mu_k](x_k^{(n)}) + \frac{\eta}{2} \|v\|_2^2 + V_\epsilon(x_k^{(n)} + v) \right\}.
    \end{equation}
    \item[--] {RWCP Subproblem:} Solve for the scalar increment $s \in \mathbb{R}$ in direction $e_{i_k}$:
    \begin{equation}\label{eqn:particle_rwcp}
    s_{k,n} \in \arg\min_{s \in \mathbb{R}} \left\{ s \partial_{i_k} \frac{\delta G}{\delta \mu}[\mu_k](x_k^{(n)}) + \frac{\eta_{i_k}}{2} s^2 + V_\epsilon(x_k^{(n)} + s e_{i_k}) \right\}.
    \end{equation}
\end{itemize}
Both subproblems are solved using 20 Newton iterations.

\textbf{Numerical Results:}
Figure~\ref{fig:rwcp_composite} reports the decay of the energy gap $E[\mu_k] - E_*$ and the barycenter norm $\|\frac{1}{N} \sum x_k^{(n)}\|_2$. RWCP exhibits consistently faster decay in both metrics. This confirms that the coordinate-wise proximal scheme successfully exploits landscape anisotropy through coordinate-specific parameters, even when the regularizer is non-separable.

\begin{figure}[htbp]
  \centering
  \includegraphics[width=0.495\textwidth]{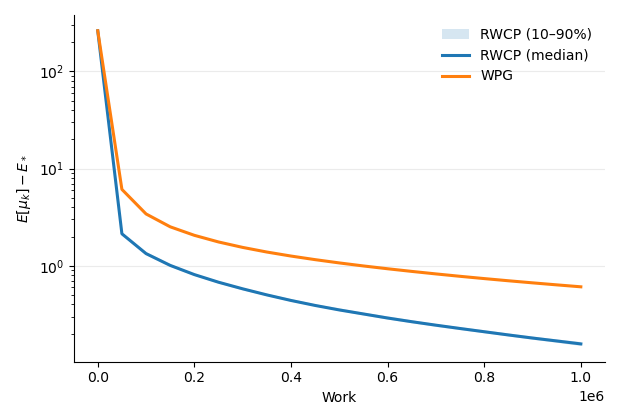}
  \includegraphics[width=0.495\textwidth]{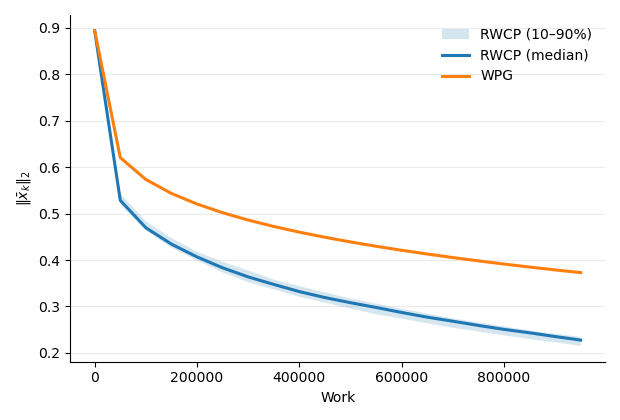}
  \caption{Example 4: RWCP vs.\ {WPG} on the composite objective $E=G+\Psi_\epsilon$. Left: energy gap $E[\mu_k]-E_*$; Right: barycenter norm. Horizontal axis: work measured in coordinate-gradient evaluations. Blue: RWCP median with $10$--$90\%$ band across $50$ runs; orange: {WPG}.}
  \label{fig:rwcp_composite}
\end{figure}

\subsection{Example 5: Two-layer Neural Network}\label{subsec:2nn}

In this final experiment, we consider two-layer neural network training. This is to train an neural network to approximate the reference function $f(x)$. Following~\cite{CB18, MMN18, SS20, DCLW22, CLTW25}, in the mean-field regime, neurons in the neural network can be viewed as drawings from a probability measure $\mu$, and the training becomes running optimization problem of the functionals
\begin{equation}\label{eqn:nn_func}
E[\mu] = \frac{1}{2}\int \bigl(g(x;\mu)-f(x)\bigr)^2\,d\pi(x) \, ,
\end{equation}
where
\begin{equation}\label{eqn:nn_gV}
g(x;\mu):=\int V(x,z)\,d\mu(z) \, , \quad\text{with}\quad V(x,z)=\alpha\,\sigma(w^\top x+b) + \beta
\end{equation}
denotes the NN generated function, with $z=(\alpha,\beta,w,b)\in\mathbb{R}\times\mathbb{R}\times\mathbb{R}^d\times\mathbb{R}$ denotes the trainable parameters (neurons) drawing from $\mu \in \mathcal{P}_2(\mathbb{R}^{d+3})$. Data $x$ is sampled from a given distribution $\pi$.

\textbf{Experimental setup.} Numerically we set the activation function as $\sigma=\tanh$. The target function has the dimension of $d=50$ and is set to be $f(x)=\tanh(w_*^\top x+0.2)+0.1$ with $(w_*)_j\propto (3/4)^{\frac{j}{2}}$ and $\|w_*\|_2=1$. The data distribution is the truncated anisotropic Gaussian $\pi=\operatorname{Law}(X\,|\,\|X\|_\infty\leq 3)$ with $X\sim\mathcal{N}(0,\Sigma)$ and $\Sigma=\operatorname{diag}(\lambda_1,\dots,\lambda_d)$, where $\lambda_j=(3/4)^{j-1}$ for $j\in[d]$. Since the reference function has a form of $\tanh$, the functional is able to attain its minimum value $E_*=0$ if $\mu$ is a Dirac measure: $\mu_*=\delta_{(1,\,0.1,\,w_*,\,0.2)}$.

In our implementation, the initial measure $\mu_0$ is supported in the bounded region $|\alpha|\leq R_\alpha$, $|\beta|\leq R_\beta$, $\|w\|_2\leq R_w$, and $|b|\leq R_b$, with $R_\alpha=R_\beta=R_b=0.3$ and $R_w=0.8$. 
We also define the working region
\begin{equation}\label{eqn:nn_region}
\mathcal{B} = \bigl\{(\alpha,\beta,w,b): |\alpha|\leq A,\ |\beta|\leq B,\ \|w\|_2\leq W,\ |b|\leq C\bigr\} \, ,
\end{equation}
where $A = 10 R_\alpha=3$, $B = 10 R_\beta = 3$, $W = 10 R_w = 8$, and $C = 10 R_b = 3$.
In all numerical runs, the particles remain inside $\mathcal{B}$.

Details of the Wasserstein gradient and smoothness analysis are deferred to Appendix~\ref{sec:nn_details}. 
In particular, the Wasserstein gradient of $E$ can be computed explicitly~\eqref{eqn:nn_grad}. One can show it is globally smooth and coordinate-smooth on $\mathcal{P}_2(\mathcal{B})$. We provide Lipschitz constants in Proposition~\ref{prop:nn_smooth}. Through this computation, we find $L = 125.033$, and $L_i$-s are plotted in Figure~\ref{fig:2nn_L_i}. 

The algorithm is initialized with $N=200$ particles sampled independently from $\alpha_n^{(0)}\sim\mathrm{Unif}[-R_\alpha,R_\alpha]$, $\beta_n^{(0)}\sim\mathrm{Unif}[-R_\beta,R_\beta]$, $b_n^{(0)}\sim\mathrm{Unif}[-R_b,R_b]$, and $w_n^{(0)}\sim\mathrm{Unif}(B_2(0,R_w))$, where $B_2(0,R_w)$ is the Euclidean ball of radius $R_w$ in $\mathbb{R}^d$. 
Integration against $\pi$ is approximated using $500$ i.i.d.\ samples. 
We compare WGD run for $2000$ iterations with RWCD run for $2000d$ coordinate updates.

\begin{figure}[htbp]
  \centering
  \includegraphics[width=0.6\textwidth]{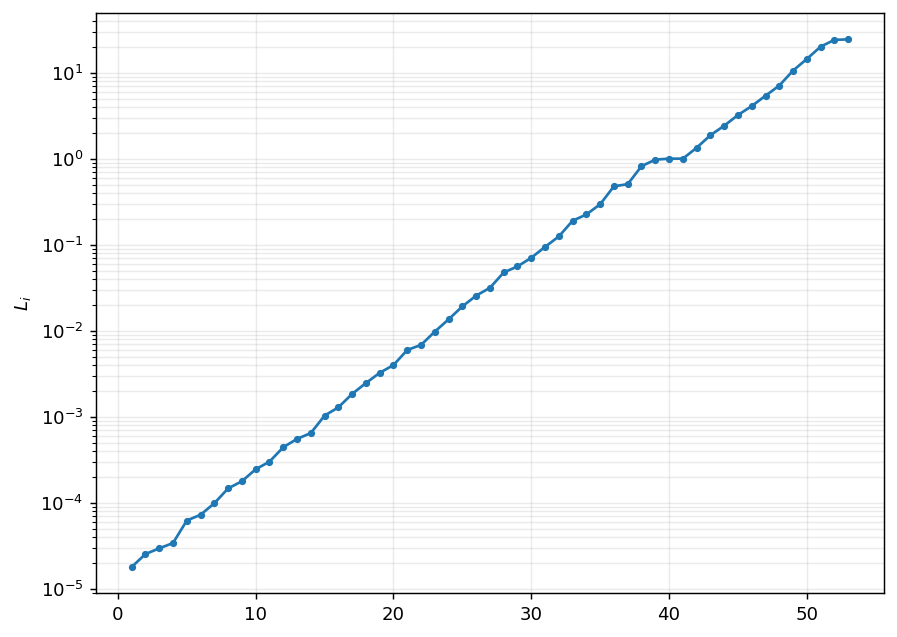}
  \caption{Example 5: Coordinate-wise smoothness constants $L_i$ of $E$ on $\mathcal{P}_2(\mathcal{B})$. The $y$-axis shows $L_i$ on a logarithmic scale, and the $x$-axis indexes the coordinates after sorting by $L_i$.}
  \label{fig:2nn_L_i}
\end{figure}

\textbf{Numerical results.}
Figure~\ref{fig:nn_rwcd} reports the decay of the energy gap $E[\mu_k]-E_*$ and the running best gradient norm $\min_{0\le i<k}\|\nabla_{\mathbb{W}}E[\mu_i]\|_{\mu_i}^2$. 
RWCD exhibits consistently faster decay in both metrics.
This suggests that the coordinate-wise scheme can effectively exploit anisotropy in the optimization landscape through coordinate-dependent parameters, even in this nonconvex neural network setting.

\begin{figure}[htbp]
  \centering
  \includegraphics[width=0.47\textwidth]{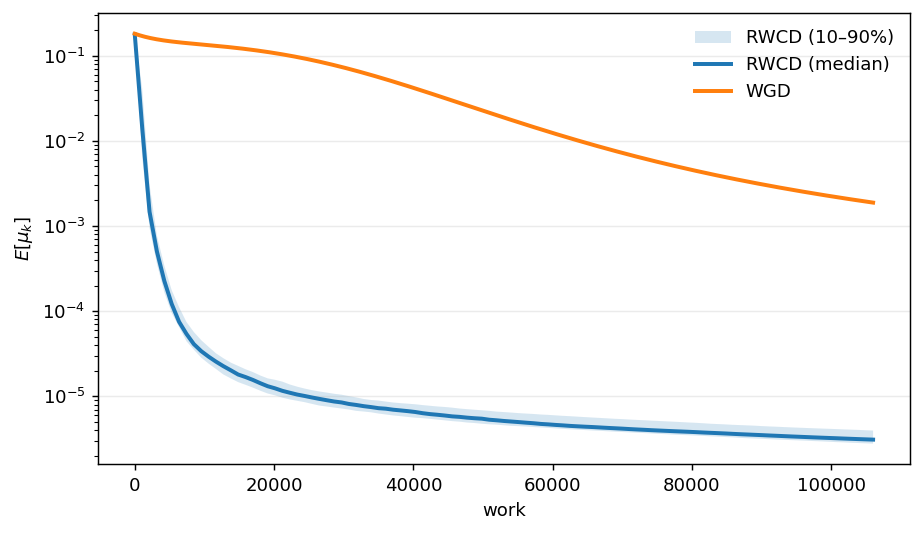}
  \includegraphics[width=0.47\textwidth]{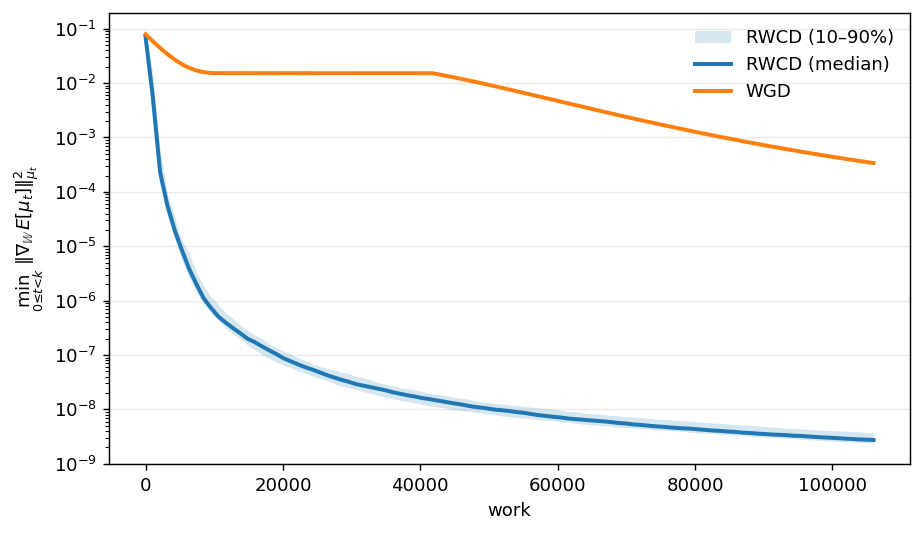}
  \caption{Example 5: RWCD vs.\ WGD on the neural network objective~\eqref{eqn:nn_func}. Left: energy gap $E[\mu_k]-E_*$. Right: running best gradient norm $\min_{0\leq i<k}\|\nabla_{\mathbb{W}}E[\mu_i]\|_{\mu_i}^2$. The horizontal axis measures work in coordinate-gradient evaluations. Blue: RWCD median with $10$--$90\%$ band over $50$ runs. Orange: WGD.}
  \label{fig:nn_rwcd}
\end{figure}

\section{Conclusion}\label{sec:conclusion}

In this work, we investigated coordinate descent algorithms in the space of probability measures $\mathcal{P}_2(\mathbb{R}^d)$. 
Inspired by the efficiency of Random Coordinate Descent (RCD) and Random Coordinate Proximal{-Gradient} (RCP) in Euclidean spaces, we introduced their counterparts in the Wasserstein setting: Random Wasserstein Coordinate Descent (RWCD) and Random Wasserstein Coordinate Proximal{-Gradient} (RWCP).

The core of our technical development lies in the realization that ``coordinates" in $\mathbb{R}^d$ correspond to marginal-distribution variations in $\mathcal{P}_2$. 
Consequently, updating coordinates to optimize a functional translates to updating these marginal distributions.
Algorithmically, this is achieved by projecting the desired pushforward maps onto individual coordinate directions. 
This projection is implemented either through direct restriction, as in RWCD, or by solving coordinate-restricted variational subproblems, as in RWCP.

Mirroring the advantages of RCD and RCP over standard Gradient Descent, particularly in high-dimensional and anisotropic landscapes, we established convergence guarantees that generalize classical Euclidean results. We provided rigorous complexity bounds for three geometric regimes: $\mathcal{O}(1/k)$ convergence to stationary points in non-convex settings, sublinear $\mathcal{O}(1/k)$ rates for geodesically convex functionals, and exponential rates under the Polyak-{\L}ojasiewicz (PL) condition or $m$-strong geodesic convexity. 
Our numerical experiments confirm these theoretical findings, demonstrating that the proposed coordinate-wise schemes significantly reduce computational costs in ill-conditioned regimes.

While completing the picture for randomized coordinate optimization solvers, the discussions presented here open the door for other coordinate-based methods in the Wasserstein space. 
In particular, examining the counterpart of Block Coordinate Descent (BCD) appears highly feasible within this framework. 
Furthermore, while we have relied on particle methods to represent probability measures numerically, other forms of approximation, such as kernel-based or grid-based methods, could be deployed. 
From a practical perspective, the scalability of our framework suggests exciting applications in large-scale machine learning where high-dimensional efficiency is paramount.

\section*{Acknowledgement}

Q.~Li thanks Professors Francois Golse and Stephen J. Wright for insightful discussions. Both Q.~Li and Y.~Xu are supported by NSF-DMS-2308440. Q.~Li is further supported by an ONR Award AWD-100997.

\newpage

\appendix

\section{Main proof for RWCD and RWCP}
We separate the discussion for RWCD and RWCP.

\subsection{Proof for RWCD}\label{sec:RWCD_proof}

In this section, we give the proof of the descent Lemma~\ref{lem:RCD_one_step} for RWCD.

\begin{proof}[Proof of Lemma~\ref{lem:RCD_one_step}]
Let $i_k$ be the coordinate sampled at iteration $k$. The update is defined by the map $T_{k} = - \gamma_{i_k} \nabla_{\mathbb{W}} E[\mu_k]$. Applying Proposition~\ref{prop:coord_wise_smooth_1st_order} with $U_i T = - \gamma_{i_k} U_{i_k} \nabla_{\mathbb{W}} E[\mu_k]$, we have:
\begin{align*}
E[\mu_{k+1}] 
&\leq E[\mu_k] - \gamma_{i_k} \langle \nabla_{\mathbb{W}}E[\mu_k], U_{i_k} \nabla_{\mathbb{W}}E[\mu_k] \rangle_{\mu_k} + \frac{L_{i_k}\gamma_{i_k}^2}{2} \left\Vert U_{i_k} \nabla_{\mathbb{W}}E[\mu_k] \right\Vert^2_{\mu_k} \\
&= E[\mu_k] - \left( \gamma_{i_k} - \frac{L_{i_k} \gamma_{i_k}^2}{2} \right) \left\Vert U_{i_k} \nabla_{\mathbb{W}}E[\mu_k] \right\Vert_{\mu_k}^2 \, .
\end{align*}
Substituting our choice of step size $\gamma_{i_k} = \frac{1}{L_{i_k}}$, we get:
\[
E[\mu_{k+1}] \leq E[\mu_k] - \frac{1}{2L_{i_k}}\left\Vert U_{i_k} \nabla_{\mathbb{W}}E[\mu_k] \right\Vert_{\mu_k}^2 \, .
\]
This implies $E[\mu_{k+1}] \leq E[\mu_k]$. Taking the expectation with respect to the random variable $i_k$:
\begin{align*}
\mathbb{E}\left[E[\mu_{k+1}] - E[\mu_k] \mid \mu_k \right] 
&= - \sum_{i=1}^d p_i \cdot \frac{1}{2 L_i} \left\Vert U_i \nabla_{\mathbb{W}}E[\mu_k] \right\Vert_{\mu_k}^2 \\
&= - \sum_{i=1}^d \left( \frac{L_i}{L_{\mathrm{sum}}} \right) \frac{1}{2 L_i} \left\Vert U_i \nabla_{\mathbb{W}} E[\mu_k] \right\Vert_{\mu_k}^2 \\
&= - \frac{1}{2L_{\mathrm{sum}}} \sum_{i=1}^d \left\Vert U_i \nabla_{\mathbb{W}} E[\mu_k] \right\Vert_{\mu_k}^2 \\
&= - \frac{1}{2 L_{\mathrm{sum}}} \left\Vert \nabla_{\mathbb{W}} E[\mu_k] \right\Vert_{\mu_k}^2 \, .
\end{align*}
This completes the proof.
\end{proof}

\subsection{Proof for RWCP}\label{sec:RWCP_proof}
A primary distinction between the analysis of RWCP and RWCD lies in the inclusion of the variational subproblem~\eqref{eqn:var_diff_grad_prox}, which introduces significant complexity to the convergence proof. To analyze this subproblem implicitly, we observe that its optimal solution must satisfy the first-order optimality conditions of the objective functional. This necessitates a careful characterization of the functional derivatives involved. While the derivatives for the linear term $\langle \nabla_{\mathbb{W}} G[\mu], U_{i} T \rangle_{\mu}$ and the quadratic penalty $\| T \|_{\mu}^2$ are straightforward, , the regularizer $\Psi$ requires a more delicate treatment.

Specifically, we consider the perturbation of $\Psi$ under the push-forward map $(\operatorname{Id} + U_{i} T)_{\#}$. For any $T \in L_2(\mu)$, we define the functional derivative of $\Psi [ (\operatorname{Id} + U_{i} T)_{\#} \mu ]$ with respect to $T$ at $T_0$ as the operator $\frac{\delta\Psi}{\delta T} [T_0]$ that satisfies the following first-order expansion for a small perturbation $\delta T$:
\begin{equation}\label{def:function_der_T_Psi}
\Psi[(\operatorname{Id} + U_{i} (T_0 + \delta T))_{\#} \mu] - \Psi[(\operatorname{Id} + U_{i} T_0)_{\#} \mu] = \int \frac{\delta\Psi}{\delta T} [T_0](x) \cdot \delta T(x) \, d\mu(x) + o(\|\delta T\|_\mu)\,.
\end{equation}

This derivative can be made explicit as is summarized in the following lemma.
\begin{lemma}\label{lem:grad_var_der_func}
Let $\Psi : \mathcal{P}_2(\mathbb{R}^d) \to \mathbb{R}$ be differentiable in the sense that its first variation $\frac{\delta \Psi}{\delta \mu}[\nu]$ exists for $\nu\in\mathcal P_2(\mathbb R^d)$ and $\nabla_{\mathbb W} \Psi[\nu]=\nabla\frac{\delta \Psi}{\delta \mu}[\nu]$ is well-defined.
Let $\mu \in \mathcal{P}_2(\mathbb{R}^d)$, $U \in \mathbb{R}^{d \times d}$, $T_0 \in \mathcal{L}_2(\mu;\mathbb{R}^d)$ be given.
Then \begin{equation}
\frac{\delta\Psi}{\delta T} [T_0] =  U^\top \nabla_{\mathbb{W}} \Psi\left[\left(\operatorname{Id} + U T_0 \right)_{\#}\mu\right]\left( x + U T_0(x)\right) \, .
\end{equation}
\end{lemma}

\begin{proof}
Consider the perturbation of $T \mapsto \Psi\left[\left( \operatorname{Id} + U T \right)_\# \mu \right]$ by $\delta T$: $$\begin{aligned}
& \Psi\left[\left( \operatorname{Id} + U \left( T_0  + \delta T \right)\right)_{\#} \mu \right] - \Psi\left[\left( \operatorname{Id} + U T_0 \right)_{\#} \mu \right] \\
= & \int_{\mathbb{R}^d} \frac{\delta \Psi}{\delta \mu}\left[\left(\operatorname{Id} + U T_0 \right)_{\#}\mu \right]\left( x \right) d\left[\left( \operatorname{Id} + U\left( T_0 + \delta T \right)\right)_{\#}\mu - \left( \operatorname{Id} + U T_0 \right)_{\#}\mu \right](x) \\
= & \int_{\mathbb{R}^d} \left(\frac{\delta \Psi}{\delta \mu}\left[\left(\operatorname{Id} + U T_0 \right)_{\#}\mu\right]\left( x + U \left(T_0(x) + \delta T(x) \right)\right) - \frac{\delta \Psi}{\delta \mu}\left[\left(\operatorname{Id} + U T_0 \right)_{\#}\mu\right]\left( x + U T_0(x)\right) \right) d \mu(x) \\
= &  \left\langle \nabla\frac{\delta \Psi}{\delta \mu}\left[\left(\operatorname{Id} + U T_0 \right)_{\#}\mu\right]\left( x + U T_0(x)\right), U \delta T(x) \right\rangle_{\mu} \\
= & \left\langle U^\top  \nabla\frac{\delta \Psi}{\delta \mu}\left[\left(\operatorname{Id} + U T_0 \right)_{\#}\mu\right]\left( x + U T_0(x)\right), \delta T(x) \right\rangle_{\mu} \,.
\end{aligned}$$ 
Comparing this with the definition~\eqref{def:function_der_T_Psi}, we have
$$\frac{\delta\Psi}{\delta T} [T_0]=  U^\top \nabla\frac{\delta \Psi}{\delta \mu}\left[\left(\operatorname{Id} + U T_0 \right)_{\#}\mu\right]\left( x + U T_0(x)\right) \, .$$
This completes the proof of the lemma.
\end{proof}

We now give a proof of Lemma~\ref{lem:descent_lemma_grad-prox} below.

\begin{proof}[Proof of Lemma~\ref{lem:descent_lemma_grad-prox}]
By the optimality of $T_k$ in~\eqref{eqn:var_diff_grad_prox}, we have \begin{equation}\label{eqn:var_minimizer_0th_cond}
\left\langle \nabla_{\mathbb{W}} G[\mu_k](x), U_{i_k} T_k(x) \right\rangle_{\mu_k}  + \frac{\eta_{i_k}}{2} \left\Vert T_k(x) \right\Vert_{\mu_k}^2 + \Psi\left[ \mu_{k+1} \right]  \leq \Psi\left[ \mu_k \right] \, .
\end{equation}
Still by optimality of $T_k$ in~\eqref{eqn:var_diff_grad_prox}, $T_k$ must be a critical point of the variational problem (by the first order optimality condition), so by Lemma~\ref{lem:grad_var_der_func}, we obtain \begin{equation}\label{eqn:crit_var_grad_prox}
U_{i_k}^\top \nabla_{\mathbb{W}} G[\mu_k](x) + \eta_{i_k} T_k(x) + U_{i_k}^\top \nabla_{\mathbb{W}} \Psi \left[\left(\operatorname{Id} + U_{i_k}T_k \right)_{\#}\mu_k\right]\left( x + U_{i_k} T_k(x)\right) = 0 \, .
\end{equation}
Now \begin{equation}\label{eqn:grad_prox_grad_movement_bd}
\begin{aligned}
& \left\Vert U_{i_k}^\top \nabla_{\mathbb{W}} E[\mu_k]\right\Vert_{\mu_{k}}^2 \\
= & \left\Vert U_{i_k}^\top \left( \nabla_{\mathbb{W}} G[\mu_k] + \nabla_{\mathbb{W}} \Psi[\mu_k] \right)\right\Vert_{\mu_{k}}^2  \\
= & \left\Vert U_{i_k}^\top \left(\nabla_{\mathbb{W}} \Psi[\mu_k](x) - \nabla_{\mathbb{W}} \Psi[\mu_{k+1}](x + U_{i_k}T_k(x))\right) - \eta_{i_k} T_k(x)\right\Vert_{\mu_{k}}^2  \\
\leq & 2 \left\Vert U_{i_k}^\top \left(\nabla_{\mathbb{W}} \Psi[\mu_k](x) - \nabla_{\mathbb{W}} \Psi[\mu_{k+1}](x + U_{i_k}T_k(x))\right)\right\Vert_{\mu_{k}}^2 +  2 \left\Vert \eta_{i_k} T_k(x)\right\Vert_{\mu_{k}}^2  \\
\leq & 2\left(H_{i_k}^2 + \eta_{i_k}^2 \right) \cdot \left\Vert T_k \right\Vert_{\mu_{k}}^2 \, ,
\end{aligned}
\end{equation} where the second equality used~\eqref{eqn:crit_var_grad_prox}; the first inequality used Young's; and the last inequality used the assumption that $\Psi$ is $H_{i_k}$-smooth along the $i_k$-th coordinate.
Then \begin{equation}\label{eqn:grad_prox_gap_movement}
\begin{aligned}
E[\mu_{k+1}] - E[\mu_k]  = & G[\mu_{k+1}] - G[\mu_k] + \Psi[\mu_{k+1}] - \Psi[\mu_k] \\
\leq & \int_{\mathbb{R}^{d}} \left\langle \nabla_{\mathbb{W}} G[\mu_k](x), U_{i_k} T_k(x) \right\rangle d \mu_{k}(x) + \frac{L_{i_k}}{2}\int_{\mathbb{R}^d} \left\Vert U_{i_k} T_k(x) \right\Vert^2 d \mu_{k}(x) \\
& + \Psi[\mu_{k+1}] - \Psi[\mu_k] \\
\leq & - \frac{\eta_{i_k} - L_{i_k}}{2}\int_{\mathbb{R}^d} \left\Vert T_k(x) \right\Vert_2^2 d \mu_k(x) \, ,
\end{aligned}
\end{equation} where in the first inequality we used Proposition~\ref{prop:coord_wise_smooth_1st_order} on $G$; and in the last inequality we used~\eqref{eqn:var_minimizer_0th_cond}.
The choice of $\eta_i$ in~\eqref{eqn:choice_grad-prox} guarantees $E[\mu_{k+1}] \leq E[\mu_k]$. Taking expectations, we obtain \begin{equation}
\begin{aligned}
& \mathbb{E}\left[ E\left(\mu_{k+1}\right) \left\vert\right. \mu_k \right] - E\left[\mu_{k}\right] \\
\leq & -\frac{1}{2} \mathbb{E}\left[ \left(\eta_{i_k} - L_{i_k}\right)\left\Vert T_k \right\Vert_{\mu_{k}}^2 \left\vert\right. \mu_k \right] \\
\leq & - \frac{1}{4} \mathbb{E}\left[ \frac{\eta_{i_k} - L_{i_k}}{H_{i_k}^2 + \eta_{i_k}^2}\left\Vert U_{i_k}^\top \nabla_{\mathbb{W}} E[\mu_k]\right\Vert_{\mu_{k}}^2 \left\vert\right. \mu_k \right] \\
= & -\frac{1}{4} \sum_{i=1}^d \frac{L_i + \sqrt{H_{i}^2 + L_{i}^2}}{\sum_{j=1}^d \left( L_j + \sqrt{H_{j}^2 + L_{j}^2} \right)} \cdot \frac{1}{2 \left( L_i + \sqrt{H_i^2 + L_i^2} \right)} \left\Vert U_{i_k}^\top \nabla_{\mathbb{W}} E[\mu_k]\right\Vert_{\mu_{k}}^2  \\
= & -\frac{1}{8 \sum_{j=1}^d \left( L_j + \sqrt{H_j^2 + L_j^2} \right)} \left\Vert \nabla_{\mathbb{W}} E[\mu_k]\right\Vert_{\mu_{k}}^2 \, ,
\end{aligned}
\end{equation} where in the first inequality we applied~\eqref{eqn:grad_prox_gap_movement}; in the second inequality we used~\eqref{eqn:grad_prox_grad_movement_bd}; and in the first inequality we used the choices of $p_i$, $\eta_i$ specified in~\eqref{eqn:choice_grad-prox}.

\end{proof}

\section{Proof of Auxiliary Results}
\subsection{Proof of Proposition~\ref{prop:coord_wise_smooth_1st_order}}\label{sec:lem:coord_wise_smooth_1st_order_proof}
Proposition~\ref{prop:coord_wise_smooth_1st_order} is a key lemma that plays an important role in the convergence proof for RWCD and RWCP.

\begin{proof}
For any $t \in [0,1]$, by the coordinate smoothness definition (Definition~\ref{def:coord_wise_smooth}), we have:
\[
\left\Vert \nabla_{\mathbb{W}} E[(\operatorname{Id} + t U_i T)_\# \mu](x + t U_i T(x)) - \nabla_{\mathbb{W}} E[\mu](x) \right\Vert_{\mu} \leq t L_i \left\Vert U_i T(x) \right\Vert_{\mu} \, .
\]
By the chain rule for the functional $E$ along the path $(\operatorname{Id} + t U_i T)_\# \mu$, we obtain:
\begin{align*}
& E[(\operatorname{Id} + U_i T)_\# \mu] - E[\mu] \\
= & \int_0^1  \left\langle \nabla_{\mathbb{W}} E[(\operatorname{Id} + t U_i T)_\# \mu](x + t U_i T(x)),  U_i T(x) \right\rangle_{\mu} \, dt \\
= & \left\langle  \nabla_{\mathbb{W}} E[\mu](x), U_i T(x) \right\rangle_{\mu} + \int_0^1  \left\langle \nabla_{\mathbb{W}} E[(\operatorname{Id} + t U_i T)_\# \mu](x_t) - \nabla_{\mathbb{W}} E[\mu](x), U_i T(x) \right\rangle_{\mu} \, dt \\
\leq & \left\langle  \nabla_{\mathbb{W}} E[\mu](x), U_i T(x) \right\rangle_{\mu} + \int_0^1  \left\Vert \nabla_{\mathbb{W}} E[(\operatorname{Id} + t U_i T)_\# \mu](x_t) - \nabla_{\mathbb{W}} E[\mu](x) \right\Vert_{\mu} \cdot \left\Vert U_i T(x) \right\Vert_{\mu} \, dt \\
\leq & \left\langle  \nabla_{\mathbb{W}} E[\mu](x), U_i T(x) \right\rangle_{\mu} + \int_0^1 t L_i \left\Vert U_i T(x) \right\Vert_{\mu}^2 \, dt \\
= & \left\langle  \nabla_{\mathbb{W}} E[\mu](x), U_i T(x) \right\rangle_{\mu} + \frac{L_i \left\Vert U_i T(x) \right\Vert_{\mu}^2}{2} \, ,
\end{align*}
where we denoted $x_t = x + t U_i T(x)$. The first inequality follows from Cauchy-Schwarz and the second from the coordinate-wise smoothness assumption.
\end{proof}

\subsection{Proof of Lemma~\ref{lem:sc_implies_PL}}\label{sec:proof_lem:sc_implies_PL}

Here we give a proof of Lemma~\ref{lem:sc_implies_PL} that links strong convexity with PL condition.

\begin{proof}
For any $\mu \in \mathcal{P}_2(\mathbb{R}^d)$, let $\pi_* \in \Pi_o(\mu_*, \mu)$.
The $m$-strong convexity of $E$ implies $$\begin{aligned}
E[\mu] - E[\mu_*] & \leq - \int \langle \nabla_{\mathbb{W}} E[\mu](y) , x - y\rangle d\pi_*(x,y) - \frac{m}{2} \mathbb{W}_2^2(\mu, \mu_*) \\
& = - \int \langle \nabla_{\mathbb{W}} E[\mu](y) , x - y\rangle d\pi_*(x,y) - \frac{m}{2} \int \Vert x - y\Vert^2 d\pi_*(x,y) \\
& \leq \frac{1}{2m} \int \Vert \nabla_{\mathbb{W}} E[\mu](y) \Vert_2^2 d \mu(y) \, ,
\end{aligned}$$
where in the first equality we used our assumption that $\pi_*$ is an optimal transport plan, and in the second inequality we applied Jensen's inequality.
\end{proof}

\subsection{Proof of Proposition~\ref{prop:examples_coord_smooth}}\label{sec:proof_prop:examples_coord_smooth}
This proposition gives four examples of energy functionals that are smooth and coordinate-wise smooth.
\begin{proof}
We prove the directional smoothness only in this proof. The overall regularity can be obtained in the same way. Fix $i\in[d]$, $\mu\in\mathcal P_2(\mathbb R^d)$, and $T\in L^2(\mu;\mathbb R^d)$, we denote $\mu_T=(\operatorname{Id}+U_iT)_\#\mu$.

\medskip\noindent\textbf{(a) Potential energy.}
For $E[\mu]=\int V\,d\mu$, one has $\frac{\delta E}{\delta\mu}[\mu](x)=V(x)$ and hence \begin{equation}\label{eqn:grad_potential}
\nabla_{\mathbb W} E[\mu](x)=\nabla V(x) \, .
\end{equation}
Therefore
\begin{align*}
&\Big\|\,U_i\big(\nabla_{\mathbb W} E[\mu_T]\circ(\operatorname{Id}+U_iT)-\nabla_{\mathbb W} E[\mu]\big)\Big\|_{\mu}^2 \\
= & \int_{\mathbb R^d}
\Big\|U_i\big(\nabla V(x+U_iT(x))-\nabla V(x)\big)\Big\|^2\,d\mu(x) \\
\leq & \int_{\mathbb R^d} L_i^2 \|U_i T(x)\|^2\,d\mu(x) \\
= & L_i^2\,\|U_i T\|_{\mu}^2 \, ,
\end{align*}
where the third line uses the $L_i$-smoothness of $V$ along the $i$-th coordinate.

\medskip\noindent\textbf{(b) Function of the mean.}
Let $m(\mu):=\int x\,d\mu(x)$ and $E[\mu]=\phi(m(\mu))$.
A direct first-variation computation yields \begin{equation}\label{eqn:grad_func_mean}
\frac{\delta E}{\delta\mu}[\mu](x)=\big\langle \nabla\phi(m(\mu)),x\big\rangle,
\qquad
\nabla_{\mathbb W} E[\mu](x)=\nabla\phi(m(\mu)),
\end{equation}
so $\nabla_{\mathbb W} E[\mu]$ is constant in $x$.
Moreover,
\[
m(\mu_T)=\int_{\mathbb R^d} (x+U_iT(x))\,d\mu(x)=m(\mu)+\int_{\mathbb R^d} U_iT(x)\,d\mu(x).
\]
Hence,
\begin{align*}
&\Big\|\,U_i\big(\nabla_{\mathbb W} E[\mu_T]\circ(\operatorname{Id}+U_iT)-\nabla_{\mathbb W} E[\mu]\big)\Big\|_{\mu}^2 \\
= &\int_{\mathbb R^d} \Big\|U_i\big(\nabla\phi(m(\mu_T))-\nabla\phi(m(\mu))\big)\Big\|^2\,d\mu(x) \\
\leq & L_i^2 \int_{\mathbb R^d} \Big\|U_i \big( m(\mu_T)-m(\mu) \big) \Big\|^2\,d\mu(x) \\
= & L_i^2 \Big\|U_i \big( m(\mu_T)-m(\mu) \big) \Big\|^2 \\
= & L_i^2 \Big\|U_i \big( \int_{\mathbb{R}^d} (x + U_i T(x)) d \mu(x) - \int_{\mathbb{R}^d} x d \mu(x) \big) \Big\|^2 \\
= & L_i^2\,\Big\|\int_{\mathbb R^d} U_iT(x)\,d\mu(x)\Big\|^2 \\
\leq & L_i^2 \int_{\mathbb R^d}\|U_iT(x)\|^2\,d\mu(x) \\
= & L_i^2\,\|U_iT\|_{\mu}^2 \, ,
\end{align*}
where the third line uses the $L_i$-smoothness of $\phi$ along the $i$-th coordinate; and the seventh line comes from Jensen's inequality.

\medskip\noindent\textbf{(c) Interaction energy.}
For $E[\mu]=\frac{1}{4}\iint W(x-y)\,d\mu(x)d\mu(y)$, we have \begin{equation}\label{eqn:grad_interaction}
\frac{\delta E}{\delta\mu}[\mu](x)=\frac{1}{2}\int_{\mathbb R^d} W(x-y)\,d\mu(y),
\qquad
\nabla_{\mathbb W} E[\mu](x)=\frac{1}{2}\int_{\mathbb R^d}\nabla W(x-y)\,d\mu(y).
\end{equation}
Using the pushforward relation $\mu_T=(\operatorname{Id}+U_iT)_\#\mu$, for $x\sim\mu$ we have
\[
\nabla_{\mathbb W} E[\mu_T](x+U_iT(x))
=\frac{1}{2}\int_{\mathbb R^d}\nabla W\big((x+U_iT(x))-(y+U_iT(y))\big)\,d\mu(y).
\]
Define
\[
\Delta(x) := U_i\Big(\nabla_{\mathbb W} E[\mu_T]\circ(\operatorname{Id}+U_iT)-\nabla_{\mathbb W} E[\mu]\Big)(x).
\]
Then, writing $T_{x,y}:=x-y$, we obtain
\begin{align*}
\|\Delta(x)\|
= & \Big\|\int_{\mathbb R^d}
U_i\Big(\nabla W\big(T_{x,y}+U_i(T(x)-T(y))\big)-\nabla W(T_{x,y})\Big)\,d\mu(y)\Big\| \\
\leq & \int_{\mathbb R^d}
\Big\|U_i\Big(\nabla W\big(T_{x,y}+U_i(T(x)-T(y))\big)-\nabla W(T_{x,y})\Big)\Big\|\,d\mu(y) \\
\leq & \int_{\mathbb R^d} L_i\,\|U_i(T(x)-T(y))\|\,d\mu(y) \\
\leq & L_i\Big(\|U_iT(x)\|+\int_{\mathbb R^d}\|U_iT(y)\|\,d\mu(y)\Big) \, ,
\end{align*}
where the second and the fourth lines use the triangle inequality; and the third line uses the $L_i$-smoothness of $W$ along the $i$-th coordinate.

Taking the square and applying Jensen's inequality, we obtain
$\big(\int \|U_iT(y)\|\,d\mu(y)\big)^2\le \int \|U_iT(y)\|^2\,d\mu(y)$:
\begin{align*}
\|\Delta(x)\|^2
\leq & 2L_i^2\|U_iT(x)\|^2
+2L_i^2\Big(\int_{\mathbb R^d}\|U_iT(y)\|\,d\mu(y)\Big)^2 \\
\leq & 2L_i^2\|U_iT(x)\|^2
+2L_i^2\int_{\mathbb R^d}\|U_iT(y)\|^2\,d\mu(y).
\end{align*}
Integrating over $x\sim\mu$ gives
\begin{align*}
\|\Delta\|_{\mu}^2
= & \int_{\mathbb R^d}\|\Delta(x)\|^2\,d\mu(x) \\
\leq & 2L_i^2\int_{\mathbb R^d}\|U_iT(x)\|^2\,d\mu(x)
+2L_i^2\int_{\mathbb R^d}\|U_iT(y)\|^2\,d\mu(y) \\
= & 4L_i^2\,\|U_iT\|_{\mu}^2.
\end{align*}
Therefore \eqref{eqn:coord_smooth_pullback} holds for the interaction energy with constant $L_i$.

\medskip\noindent\textbf{(d) MMD-type kernel energy.}
For $\sigma:=\mu-\nu$ and
$E[\mu]=\frac{1}{2}\iint k(x-y)\,d\sigma(x)d\sigma(y)$, we have \begin{equation}\label{eqn:grad_MMD}
\frac{\delta E}{\delta\mu}[\mu](x)=\int_{\mathbb R^d} k(x-y)\,d\sigma(y),
\qquad
\nabla_{\mathbb W} E[\mu](x)=\int_{\mathbb R^d}\nabla k(x-y)\,d\sigma(y).
\end{equation}
Proceeding as in part (c) with $k$ in place of $W$ and $\sigma$ in place of $\mu$, we obtain for $x\sim\mu$ the bound
\[
\|\Delta(x)\|\leq L_i\Big(\|U_iT(x)\|+\int_{\mathbb R^d}\|U_iT(y)\|\,d|\sigma|(y)\Big),
\]
where $|\sigma|$ is the total variation measure.
Since $\nu$ is fixed and $\mu$ is a probability measure, using $|\sigma|(\mathbb R^d)\leq 2$ and $\int \|U_iT\|^2\,d|\sigma|\le 2\int \|U_iT\|^2\,d\mu$. With the same process, we apply Young's and Jensen as was done in part (c), we obtain
\[
\|\Delta\|_{\mu}^2 \leq 16L_i^2\,\|U_iT\|_{\mu}^2,
\]
so \eqref{eqn:coord_smooth_pullback} holds for the MMD-type kernel energy with constant $4L_i$.
\end{proof}

\section{Convergence of RCP in Euclidean Space}\label{sec:proof_rcp_euclidean}

In this section, we provide the proof for the convergence rate of RCP summarized in Section~\ref{sec:euclid_rcd_rcp} for the Euclidean space. Many presented results are extensions from~\cite{CN23}. The main difference is that, instead of being generic, we deploy the optimal coordinate sampling rate.

\begin{theorem}[Convergence of RCP]\label{thm:rcp_conv_euclid}
Let $f(x) = g(x) +\psi(x)$, where $g$ and $\psi$ are differentiable, $L_i$-smooth and $H_i$-smooth along the i-th coordinate respectively. Let RCP be defined as in~\eqref{eq:euclid_rcp} and~\eqref{eq:euclid_rcp_subprob}, and let $\mathcal{C}$ be defined as in~\eqref{eqn:calC}.
Choose the sampling probabilities and step sizes as in~\eqref{eqn:choice_grad-prox}. Then $f(x_{k+1}) \leq f(x_k)$ for all $k$, and RCP converges at the following rate:
\begin{itemize}
    \item \textbf{Non-convex:} The algorithm finds an $\epsilon$-approximate stationary point in $\mathcal{O}(\frac{\mathcal{C}}{\epsilon})$ iterations:
    $$ \mathbb{E}\left[\min_{0 \leq t < k} \|\nabla f(x_t)\|^2\right] \leq \frac{\mathcal{C}\left(f(x_0) - f(x_*) \right)}{k} \, .$$
    \item \textbf{Convex:}  If $f$ is convex and satisfies Assumption~\ref{asp:euclid_bls} with radius $R_0$ at initialization $x_0$, then
    $$ \mathbb{E}[f(x_k) - f(x_*)] \leq \frac{\mathcal{C} R_0^2}{k} \, ,$$ meaning, the algorithm finds an optimality gap of $\epsilon$ in $\mathcal{O}(\frac{\mathcal{C}}{\epsilon})$ iterations.
    \item \textbf{$m$-strongly convex or $m$-PL} if $f$ is $m$-strongly convex or $m$-PL, then
    $$ \mathbb{E}[f(x_k) - f(x_*)] \leq \left(1 - \frac{2m}{\mathcal{C}} \right)^k (f(x_0) - f(x_*)) \, ,$$ meaning, the algorithm finds an optimality gap of $\epsilon$ in $\mathcal{O}(\frac{\mathcal{C}}{m} \ln(\frac{1}{\epsilon}))$ iterations.

\end{itemize}
\end{theorem}

The proof of the theorem strongly relies on the following proposition.

\begin{proposition}[Coordinate-wise descent]\label{lem:coord_wise_smooth_1st_order_euclid}
Let $f : \mathbb R^d \to \mathbb R$ be differentiable. Fix $x \in \mathbb R^d$ and $i\in[d]$. Assume $f$ is $L_i$-smooth along the $i$-th coordinate at $x$ in the sense of Definition~\ref{def:coord_smooth_euclid}, then for any $s \in \mathbb{R}$,
\begin{equation}
f(x + s e_i) \leq f(x) + \partial_{i} f(x) s + \frac{L_i}{2}s^2 \, .
\end{equation}
\end{proposition}

This is a standard result and can be found in many literature. See e.g.~\cite{N12, WR22}.

We can now show the one step decent lemma of RCP.

\begin{lemma}[Descent Lemma for RCP]\label{lem:descent_lemma_grad-prox_euclid}
Under the same conditions as in Theorem~\ref{thm:rcp_conv_euclid}, we denote $(x_k)_{k \geq 0}$ the sequence generated by RCP using the prescribed probability sampling rate and step sizes, then for any $k \geq 0$, $f(x_{k+1}) \leq f(x_k)$, and
\begin{equation}\label{eqn:rcp_onestep_euclid}
\mathbb{E}\left[f(x_{k+1} - f(x_k)) \mid x_k \right] \leq - \frac{1}{\mathcal{C}} \left\Vert \nabla f(x_k) \right\Vert^2 \, ,
\end{equation}
where $\mathcal{C}$ is defined in~\eqref{eqn:calC}.
\end{lemma}

\begin{proof}
By the optimality of $s_k$ in~\eqref{eq:euclid_rcp_subprob}, we have \begin{equation}\label{eqn:var_minimizer_0th_cond_euclid}
\partial_{i_k}g(x_k) s_k + \frac{\eta_{i_k}}{2} s_{k}^2  + \psi(x_{k+1})  \leq \psi(x_k) \,,
\end{equation}
where the right hand side can be viewed as setting $s=0$. Meanwhile, since $s_k$ is the optimal value obtained in~\eqref{eq:euclid_rcp_subprob}, it must be a critical point of the variational problem, so we obtain \begin{equation}\label{eqn:crit_var_grad_prox_euclid}
\partial_{i_k}g(x_k) + \eta_{i_k} s_{k} + \partial_{i_k} \psi(x_{k+1}) = 0 \, .
\end{equation}
Then:
\begin{equation}\label{eqn:grad_prox_grad_movement_bd_euclid}
\begin{aligned}
\left\Vert \partial_{i_k} f(x_k) \right\Vert^2 
= & \left\Vert \partial_{i_k} g(x_k) + \partial_{i_k} \psi(x_k) \right\Vert_{\mu_{k}}^2  \\
= & \left\Vert \partial_{i_k} \psi(x_k) - \partial_{i_k} \psi(x_{k+1}) - \eta_{i_k} s_{k} \right\Vert^2  \\
\leq & 2 \left\Vert \partial_{i_k} \psi(x_k) - \partial_{i_k} \psi(x_{k+1}) \right\Vert^2 +  2 \left\Vert \eta_{i_k} s_{k}\right\Vert^2  \\
\leq & 2\left(H_{i_k}^2 + \eta_{i_k}^2 \right) \cdot \left\Vert s_k \right\Vert^2 \, ,
\end{aligned}
\end{equation}
where we used~\eqref{eqn:crit_var_grad_prox_euclid}, Young's, and the assumption the smoothness of $\psi$. Therefore:
\begin{equation}\label{eqn:grad_prox_gap_movement_euclid}
\begin{aligned}
f(x_{k+1}) - f(x_k)  = & g(x_{k+1}) - g(x_k) + \psi(x_{k+1}) - \psi(x_k) \\
\leq & \partial_{i_k} g(x_k) s_k + \frac{L_{i_k}}{2} s_k^2 + \psi(x_{k+1}) - \psi(x_k) \\
\leq & - \frac{\eta_{i_k} - L_{i_k}}{2} s_k^2 \, ,
\end{aligned}
\end{equation}
where we used Proposition~\ref{lem:coord_wise_smooth_1st_order_euclid} and the inequality~\eqref{eqn:var_minimizer_0th_cond_euclid}.
The choice of $\eta_i$ in~\eqref{eqn:choice_grad-prox} guarantees $f(x_{k+1}) \leq f(x_k)$. Taking expectations, we obtain \begin{equation}
\begin{aligned}
\mathbb{E}\left[ f(x_{k+1}) \left\vert\right. x_k \right] - f(x_k) \leq & -\frac{1}{2} \mathbb{E}\left[ \left(\eta_{i_k} - L_{i_k}\right) s_k^2 \left\vert\right. x_k \right] \\
\leq & - \frac{1}{4} \mathbb{E}\left[ \frac{\eta_{i_k} - L_{i_k}}{H_{i_k}^2 + \eta_{i_k}^2}\left\Vert \partial_{i_k} f(x_k) \right\Vert^2 \left\vert\right. x_k \right] \\
= & -\frac{1}{8 \sum_{j=1}^d \left( L_j + \sqrt{H_j^2 + L_j^2} \right)} \left\Vert \nabla f(x_k) \right\Vert^2 \, ,
\end{aligned}
\end{equation}
where we plugged in the choice of $p_i$, $\eta_i$ specified in~\eqref{eqn:choice_grad-prox}.
\end{proof}

Now we are ready to give the proof of Theorem~\ref{thm:rcp_conv_euclid}. We note the proof is structurally the same as that for Theorem~\ref{thm:rcp_conv}.

\begin{proof}[Proof of Theorem~\ref{thm:rcp_conv_euclid}]
\begin{itemize}
\item[--] In the nonconvex case. By~\eqref{eqn:rcp_onestep_euclid}, $$\left\Vert \nabla f(x_k) \right\Vert^2 \leq \mathcal{C} \left( f(x_k) - \mathbb{E}\left[ f(x_{k+1}) \left\vert\right. x_k \right] \right) \,.$$
Summing all iterations up, one gets
$$\begin{aligned}
\mathbb{E}\left[\min_{\{0 \leq t < k\}} \left\Vert \nabla f(x_t) \right\Vert^2\right] 
\leq &  \mathbb{E}\left[\frac{\sum_{t=0}^{k-1} \left\Vert \nabla f(x_t) \right\Vert^2}{k}\right] \\
\leq & \frac{\mathcal{C}}{k} \sum_{t=0}^{k-1}\mathbb{E} \left[\mathbb{E}\left[ f(x_t) - f(x_{t+1}) \left\vert \right. x_t  \right] \right] \\
\leq & \frac{\mathcal{C}}{k} \mathbb{E}\left[f(x_0) - f(x_k)\right] \\
\leq & \frac{\mathcal{C}}{k} \left( f(x_0) - f(x_*) \right) \, .
\end{aligned}$$
\item[--] In the convex case:
\[
\begin{aligned}
f(x_k) - f(x_*) \leq &  (\nabla f(x_k))^\top ( x_k - x_* )  \\
\leq & \| \nabla f(x_k) \| \cdot \| x_k - x_*  \| \\ 
\leq & R_0 \| \nabla f(x_k) \| \, ,
\end{aligned}
\]
where we used Cauchy-Schwarz inequality, Definition~\ref{def:euclid_convex_pl} and Assumption~\ref{asp:euclid_bls}. 
Therefore, conditioned on $x_k$, utilizing~\eqref{eqn:rcp_onestep_euclid},
$$\begin{aligned}
& \mathbb{E}\left[ f(x_{k+1}) -f(x_*) \right] \\
\leq & f(x_k) -f(x_*) - \frac{1}{\mathcal{C}}\left\Vert \nabla f(x_k) \right\Vert^2  \\
\leq & f(x_k) -f(x_*) - \frac{\left(f(x_k) -f(x_*)\right)^2}{\mathcal{C} R_0^2} \, .
\end{aligned}$$
Denoting $\phi_k := \mathbb{E}[f(x_k) -f(x_*)]$, this inequality translates to $\phi_{k+1} \leq \phi_k - \frac{\phi_k^2}{\mathcal{C} R_0^2}$. With the same arguments as in~\eqref{eq:phi_sequel}, one conclude the analysis.
\item[--] In the $m$-strongly convex or $m$-PL case.
By~\eqref{eqn:rcp_onestep_euclid},
$$\mathbb{E}\left[ f(x_{k+1}) \left\vert\right. x_k \right] - f(x_k) \leq - \frac{1}{\mathcal{C}} \left\Vert \nabla f(x_k) \right\Vert^2 \leq -  \frac{2m}{\mathcal{C}} \left(f(x_k) - f(x_*) \right)  \, .$$
Subtracting $f[x_*]$ from both sides and reorganizing the terms, we obtain $$\mathbb{E}\left[ f(x_{k+1}) \left\vert\right. x_k \right] - f(x_*) \leq \left(1 - \frac{2m}{\mathcal{C}}  \right) \cdot \left( f(x_k) - f(x_*) \right) \, .$$
The conclusion follows by induction.
\end{itemize}

\end{proof}

\section{Details for Example 5}\label{sec:nn_details}

Here we present the details of the Wasserstein gradient and smoothness analysis for Example~5 in Section~\ref{subsec:2nn}. 
Taking the first variation of~\eqref{eqn:nn_func}, we obtain
\begin{equation}\label{eqn:first_variation_E_nn}
\frac{\delta E}{\delta \mu}[\mu](z)
 = \int r_\mu(x)  V(x,z) \, d\pi(x) \, , \quad r_\mu(x) := g(x;\mu) - f(x) \,,
\end{equation}
so the Wasserstein gradient of $E$ is \begin{equation}\label{eqn:nn_grad}
G[\mu](z) := \nabla_{\mathbb{W}} E[\mu](z) = \nabla_z \frac{\delta E}{\delta \mu}[\mu](z) = \int r_\mu(x)\phi(z;x)
d\pi(x) \, , 
\end{equation}
where:
\begin{equation}\label{eqn:phi_definition}
\phi(z;x):=\nabla_z V(x,z) = \begin{bmatrix}
\sigma(s)\\
1\\
\alpha \sigma'(s)x\\
\alpha \sigma'(s)
\end{bmatrix}\, , \qquad s=w^\top x+b \, .
\end{equation}

To state the corresponding smoothness bounds, we introduce the auxiliary functions
$K,S:\operatorname{supp}(\pi)\to\mathbb{R}$ defined by
\begin{equation}\label{eqn:nn_aux_defs}
K(x):=\sqrt{1+\|x\|_2^2} \, ,\qquad S(x):=W\|x\|_2+C \,
\end{equation}
and
\begin{equation}\label{eqn:nn_local_sup_bounds}
m_0(x):=\sup_{|u|\leq S(x)}|\sigma(u)| \, , \quad m_1(x) := \sup_{|u|\leq S(x)}|\sigma'(u)| \, , \quad m_2(x) := \sup_{|u| \leq S(x)}|\sigma''(u)| \, .
\end{equation}
Since $\sigma=\tanh$, these become \[  m_0(x)=\tanh(S(x)) \, , \qquad m_1(x) = 1 \, , \qquad m_2(x)=\sup_{|u| \leq S(x)} 2|\tanh(u)|(1-\tanh^2(u)) \, . \]
The smoothness of $E$ on $\mathcal{P}_2(\mathcal{B})$ is characterized by the following proposition:

\begin{proposition}\label{prop:nn_smooth}
Let $E$ be defined by~\eqref{eqn:nn_func} and \eqref{eqn:nn_gV}, with $\sigma=\tanh$, and let $\mathcal{B}\subset \mathbb{R}^{d+3}$ be the bounded region defined in~\eqref{eqn:nn_region}. 
Then $E$ is $L$-smooth on $\mathcal{P}_2(\mathcal{B})$, with
\begin{equation}\label{eqn:nn_L_global}
L = \int_{\mathbb{R}^d} \Bigl(M_0(x)^2 + R(x)M_1(x)\Bigr)\,d\pi(x) \, ,
\end{equation}
where
\begin{subequations}
\begin{align}
R(x) & := A\,m_0(x)+B+|f(x)| \, , \label{eqn:nn_Rx}\\
M_0(x) & := \sup_{|u|\leq S(x)} \sqrt{1+\sigma(u)^2+A^2\sigma'(u)^2K(x)^2} \, , \label{eqn:nn_M0}\\
M_1(x)& := \frac{K(x)}{2}\,\sup_{|u|\leq S(x)}\bigl\{ A|\sigma''(u)|K(x) + \sqrt{A^2\sigma''(u)^2K(x)^2+4\sigma'(u)^2} \bigr\} \, .
\label{eqn:nn_M1}
\end{align}
\end{subequations}

Moreover, $E$ is coordinate-wise smooth on $\mathcal{P}_2(\mathcal{B})$. Denote \(L_\alpha\), $L_\beta$, $L_b$ and $L_{w_i}$ smoothness parameter along the $\alpha$, \(\beta\), $b$ and $w_i$, we have:
\begin{subequations}\label{eqn:nn_L_coord}
\begin{align}
L_\alpha & = \int_{\mathbb{R}^d} m_0(x)^2\,d\pi(x) \, , \qquad L_\beta = 1 \, , \label{eqn:nn_Lalpha_beta}\\
L_b & = \int_{\mathbb{R}^d} \Bigl(A^2m_1(x)^2 + A\,R(x)m_2(x)\Bigr)\,d\pi(x), \label{eqn:nn_Lb} \\
L_{w_i} & = \int_{\mathbb{R}^d} x_i^2\Bigl(A^2m_1(x)^2 + A\,R(x)m_2(x)\Bigr)\,d\pi(x) \, , \qquad i\in[d].
\label{eqn:nn_Lwi}
\end{align}
\end{subequations}
Here $K,S,m_0,m_1,m_2$ are defined in~\eqref{eqn:nn_aux_defs} and \eqref{eqn:nn_local_sup_bounds}.
\end{proposition}

\begin{lemma}
    Under the same assumption as in Proposition~\ref{prop:nn_smooth}, we have, $\forall z_1,z_2\in\mathcal{B}$:
\begin{equation}\label{eqn:nn_V_lip}
\sup_{z\in\mathcal{B}}\|\phi(z;x)\|_2=M_0(x)\,,\quad\text{and}\quad |V(x,z_1)-V(x,z_2)| \leq M_0(x)\|z_1-z_2\|_2 \,,
\end{equation}
and
\begin{equation}\label{eqn:nn_phi_lip} \|\phi(z_1;x)-\phi(z_2;x)\|_2 \leq M_1(x)\,\|z_1-z_2\|_2 \, . \end{equation}
\end{lemma}
\begin{proof}
We first prove the global $L$-smoothness. For any $z=(\alpha,\beta,w,b)\in\mathcal{B}$, noting the definition~\eqref{eqn:phi_definition}:
\[ |s| \leq \|w\|_2\|x\|_2+|b| \leq W\|x\|_2 + C = S(x) \, . \]
Therefore $\sup_{z\in\mathcal{B}}\|\phi(z;x)\|_2=M_0(x)$,
indicating that $V(x,\cdot)$ is $M_0(x)$-Lipschitz on $\mathcal{B}$. This proves~\eqref{eqn:nn_V_lip}.

Next we estimate the Lipschitz constant of $\phi(\cdot;x)$. This amounts to compute $D_z\phi(z;x)$ for a given fixed $x$. In a perturbation direction $u=(\delta_\alpha,\delta_\beta,\delta_w,\delta_b)$:
\[
D_z\phi(z;x)[u] = \left.\frac{d}{dt}\right|_{t=0}\phi(z+tu;x) = \begin{bmatrix} \sigma'(s)q \\ 0 \\ \bigl(\sigma'(s)\delta_\alpha+\alpha\sigma''(s)q\bigr)x \\ \sigma'(s)\delta_\alpha+\alpha\sigma''(s)q \end{bmatrix} \,,
\]
where we used the notation $q:=x^\top\delta_w+\delta_b$. Therefore,
\[\begin{aligned} \|D_z\phi(z;x)[u]\|_2^2 = & \sigma'(s)^2 q^2 + \bigl(\sigma'(s)\delta_\alpha+\alpha\sigma''(s)q\bigr)^2 \bigl(\|x\|_2^2+1\bigr) \\ = &  \sigma'(s)^2 q^2 + K(x)^2\bigl(\sigma'(s)\delta_\alpha+\alpha\sigma''(s)q\bigr)^2 \\
= & \left\| \begin{bmatrix} 0 & \sigma'(s)K(x)\\ \sigma'(s)K(x) & \alpha\sigma''(s)K(x)^2 \end{bmatrix} \begin{bmatrix} \delta_\alpha \\ \frac{q}{K(x)} \end{bmatrix} \right\|_2^2 \, . \end{aligned}
\] 
It follows that
\begin{align}
\|D_z\phi(z;x)\|_{\operatorname{op}} &\leq \left\| \begin{bmatrix} 0 & \sigma'(s)K(x)\\ \sigma'(s)K(x) & \alpha\sigma''(s)K(x)^2 \end{bmatrix} \right\|_{\operatorname{op}}\\
&\leq \frac{K(x)}{2} \Bigl( |\alpha|\,|\sigma''(s)|\,K(x) + \sqrt{\alpha^2\sigma''(s)^2K(x)^2+4\sigma'(s)^2} \Bigr)\,.
\end{align}
Consequently, taking the supremum over $z\in\mathcal{B}$ and using $|\alpha|\leq A$ together with $|s|\leq S(x)$, we obtain
\[ \sup_{z\in\mathcal{B}} \|D_z\phi(z;x)\|_{\operatorname{op}} \leq M_1(x) \, , \] where $M_1(x)$ is defined in~\eqref{eqn:nn_M1}. 
Hence $\phi(\cdot;x)$ is $M_1(x)$-Lipschitz on $\mathcal{B}$ and we conclude~\eqref{eqn:nn_phi_lip}.
\end{proof}

\begin{proof}[Proof for Proposition~\ref{prop:nn_smooth}]
Denote $\nu=(\operatorname{Id}+ T)_\#\mu$, where $T\in L^2(\mu; \mathbb{R}^{d+3})$ and $\mu, \nu\in\mathcal{P}_2(\mathcal{B})$. To show global $L$-smoothness, we are to prove that there is a constant $L$ so that:
\begin{equation}\label{eqn:Lip_nn_L}
    \bigl\|G[\nu]\circ(\mathrm{Id}+T)-G[\mu] \bigr\|_{\mu}
\leq L\|T\|_{\mu} \,,
\end{equation}
where $G$ is defined in~\eqref{eqn:nn_grad}. To do so, we first expand:
\begin{align}\label{eqn:expansion_NN_G}
& G[\nu](z+T(z))-G[\mu](z) \\
= & \underbrace{\int_{\mathbb{R}^d} \bigl(r_\nu(x)-r_\mu(x)\bigr)\phi(z+T(z);x)\,d\pi(x)}_{\mathrm{Term I}} + \underbrace{\int_{\mathbb{R}^d} r_\mu(x)\bigl(\phi(z+T(z);x)-\phi(z;x)\bigr)\,d\pi(x)}_{\mathrm{Term II}} \, .\nonumber
\end{align}
We now analyze the two terms separately:
\begin{description}
    \item[Term I:] According to~\eqref{eqn:first_variation_E_nn}, we have:
\[ |r_\nu(x)-r_\mu(x)| = |g(x;\nu)-g(x;\mu)| =
\left| \int \bigl(V(x,z+T(z))-V(x,z)\bigr)\,d\mu(z) \right| \leq M_0(x)\|T\|_{\mu} \, ,
\]
where we used the Lipschitz condition in~\eqref{eqn:nn_V_lip} and Jensen. Noting again~\eqref{eqn:nn_V_lip} to bound $\phi$:
\[
\left\|\int_{\mathbb{R}^d} \bigl(r_\nu(x)-r_\mu(x)\bigr)\phi(z+T(z);x)\,d\pi(x) \right\|_2 \leq \left(\int_{\mathbb{R}^d}M_0(x)^2\,d\pi(x)\right)\|T\|_{\mu} \, .
\]
    \item[Term II:] Noting $\mu$ is supported in $\mathcal{B}$,
\[ |g(x;\mu)| \leq \int \bigl(|\alpha|\,|\sigma(s)|+|\beta|\bigr)\,d\mu(z) \leq A\,m_0(x)+B \, , \]
and hence
\begin{equation}\label{eqn:r_mu_R_bd}
|r_\mu(x)| \leq A\,m_0(x)+B+|f(x)| = R(x) \, .
\end{equation}
We obtain, using~\eqref{eqn:nn_phi_lip}:
\[
\left\| \int_{\mathbb{R}^d} r_\mu(x)\bigl(\phi(z+T(z);x)-\phi(z;x)\bigr)\,d\pi(x) \right\|_2
\leq
\left(\int_{\mathbb{R}^d}R(x)M_1(x)\,d\pi(x)\right)|T(z)| \, .
\]
\end{description}
Combine these two norms and take the $\|\cdot\|_{\mu}$ norm:
\[
\bigl\|G[\nu]\circ(\mathrm{Id}+T)-G[\mu] \bigr\|_{\mu}
\leq
\left(\int_{\mathbb{R}^d}M_0(x)^2\,d\pi(x) + \int_{\mathbb{R}^d}R(x)M_1(x)\,d\pi(x) \right)\|T\|_{\mu} \, .
\]
This concludes~\eqref{eqn:Lip_nn_L} with $L$ defined in~\eqref{eqn:nn_L_global}.

To demonstrate the coordinate-wise smoothness, one proceeds similarly.
\begin{description}
    \item[--] \textbf{The $\alpha$-coordinate.}
Let $\nu=(\operatorname{Id}+U_\alpha T)_\#\mu$, where $T\in L^2(\mu;\mathbb{R}^{d+3})$ and $\mu,\nu\in\mathcal{P}_2(\mathcal{B})$. We are to control:
\[
\|U_\alpha G[\nu](z+U_\alpha T(z)) - U_\alpha G[\mu](z)\|_\mu
\]
where by~\eqref{eqn:nn_grad},
\[
U_\alpha G[\mu](z) = \int_{\mathbb{R}^d} r_\mu(x)\sigma(s)\,d\pi(x) \, .
\]
Similar to what is done in~\eqref{eqn:expansion_NN_G}, this term is split into two. The analysis for Term II is exactly the same, and we focus on Term I below.

Since perturbing only the $\alpha$-coordinate does not change $s=w^\top x+b$, we have
\[
\left|V(x,z+U_\alpha T(z))-V(x,z)\right| = |U_\alpha T(z)|\,\cdot |\sigma(s)| \leq m_0(x)\,\cdot |U_\alpha T(z)| \, .
\]
Thus by definition~\eqref{eqn:nn_gV}:
\[
\left|g(x;\nu)-g(x;\mu)\right|=
\left| \int \bigl(V(x,z+U_\alpha T(z))-V(x,z)\bigr)\,d\mu(z) \right| \leq m_0(x) \cdot \|U_\alpha T\|_\mu \, .
\]
Therefore,
\[
\left|U_\alpha G[\nu](z+U_\alpha T(z)) - U_\alpha G[\mu](z)\right| \leq \left(\int_{\mathbb{R}^d}m_0(x)^2\,d\pi(x)\right)\|U_\alpha T\|_\mu \, .
\]
Taking the $\|\cdot\|_\mu$ norm, we conclude that $E$ is $L_\alpha$-smooth along the $\alpha$-coordinate.

\item[--] \textbf{The $\beta$-coordinate.}
Let $\nu=(\operatorname{Id}+U_\beta T)_\#\mu$, where $T\in L^2(\mu;\mathbb{R}^{d+3})$ and $\mu,\nu\in\mathcal{P}_2(\mathcal{B})$.
From \eqref{eqn:nn_grad},
\[
U_\beta G[\mu](z) = \int_{\mathbb{R}^d} r_\mu(x)\,d\pi(x) \, .
\]
Since perturbing only the $\beta$-coordinate changes $V(x,z)$ by exactly $U_\beta T(z)$, we have
\[
\left|V(x,z+U_\beta T(z))-V(x,z)\right| = |U_\beta T(z)| \, ,
\]
and hence
\[
\left|g(x;\nu)-g(x;\mu)\right| \leq \|U_\beta T\|_\mu \, .
\]
Therefore,
\[
\left|U_\beta G[\nu](z+U_\beta T(z)) - U_\beta G[\mu](z)\right| \leq \|U_\beta T\|_\mu \, .
\]
Taking the $\|\cdot\|_\mu$ norm, we obtain that $E$ is $L_\beta$-smooth along the $\beta$-coordinate.
\item[--] \textbf{The $b$-coordinate.}
Let $\nu=(\operatorname{Id}+U_b T)_\#\mu$, where $T\in L^2(\mu;\mathbb{R}^{d+3})$ and $\mu,\nu\in\mathcal{P}_2(\mathcal{B})$. We are to control
\begin{align*}
\|U_b G[\nu](z+U_bT(z)) - U_b G[\mu](z)\|_\mu\,.
\end{align*}
Similar to what is done in~\eqref{eqn:expansion_NN_G}, we split this term into:
\begin{align*}
&|U_b G[\nu](z+U_bT(z)) - U_b G[\mu](z)| \\
\leq & \int_{\mathbb{R}^d}
\left|r_\nu(x)-r_\mu(x)\right|\,\cdot\left|\alpha\sigma'(s+U_bT(z))\right|\,d\pi(x)\,.
\end{align*}
Here we used the fact that $U_b G[\mu](z) = \int_{\mathbb{R}^d} r_\mu(x)\,\alpha\sigma'(s)\,d\pi(x)$. By straightforward computation, we have:
\[
\left|g(x;\nu)-g(x;\mu)\right| \leq A\, \cdot m_1(x) \cdot \|U_bT\|_\mu \,,
\]
and
\[
\left|\alpha\sigma'(s+U_bT(z))\right| \leq A\, \cdot m_1(x) \, ,
\]
and
\[
\left|\alpha\sigma'(s+U_bT(z))-\alpha\sigma'(s)\right|  \leq A\, \cdot m_2(x)\, \cdot |U_bT(z)| \, .
\]
Taking the $\|\cdot\|_\mu$ norm, we conclude that $E$ is $L_b$-smooth along the $b$-coordinate.
\item[--]\textbf{The $w_i$-coordinate.}
Let $\nu=(\operatorname{Id}+U_{w_i}T)_\#\mu$, where $T\in L^2(\mu;\mathbb{R}^{d+3})$ and $\mu,\nu\in\mathcal{P}_2(\mathcal{B})$.
From \eqref{eqn:nn_grad},
\[
U_{w_i} G[\mu](z) = \int_{\mathbb{R}^d} r_\mu(x)\,\alpha\sigma'(s)x_i\,d\pi(x) \, .
\]
Since perturbing only the $w_i$-coordinate replaces $s$ by $s+U_{w_i}T(z)\,x_i$, we have
\[
\left|V(x,z+U_{w_i}T(z))-V(x,z)\right| \leq A\, \cdot m_1(x)\, \cdot |x_i|\, \cdot |U_{w_i}T(z)| \, ,
\]
and hence
\[
\left|g(x;\nu)-g(x;\mu) \right| \leq A\,\cdot m_1(x)\,\cdot |x_i|\,\cdot \|U_{w_i}T\|_\mu \, .
\]
Moreover,
\[
\left|\alpha\sigma'(s+U_{w_i}T(z)\, \cdot x_i) \cdot x_i \right| \leq A\, \cdot m_1(x)\, \cdot |x_i| \, ,
\]
and
\[
\left|\alpha\sigma'(s+U_{w_i}T(z)\, \cdot x_i) \cdot x_i-\alpha\sigma'(s) \cdot x_i\right| \leq A\, \cdot m_2(x)\, \cdot x_i^2\, \cdot |U_{w_i}T(z)|.
\]
Again using the splitting in~\eqref{eqn:expansion_NN_G},
\begin{align*}
&\left|U_{w_i} G[\nu](z+U_{w_i}T(z)) - U_{w_i} G[\mu](z)\right| \\
\leq & \left( \int_{\mathbb{R}^d} x_i^2 \cdot \bigl(A^2m_1(x)^2+A\,R(x) \cdot m_2(x)\bigr)\,d\pi(x) \right)\|U_{w_i}T\|_\mu \, .
\end{align*}
Taking the $\|\cdot\|_\mu$ norm, we obtain that $E$ is $L_{w_i}$-smooth along the $w_i$-coordinate for each $i\in[d]$.
\end{description}
This concludes the entire proposition.
\end{proof}

\bibliographystyle{plain}
\bibliography{main_ref}

\end{document}